\documentclass[11pt]{airesearch}
\usepackage[T1]{fontenc}
\usepackage{XCharter}

\usepackage{fullpage}
\usepackage{wrapfig}
\usepackage{mathtools}
\setlist[itemize]{leftmargin=*,topsep=1pt,itemsep=1pt}
\setlist[enumerate]{leftmargin=*,topsep=1pt,itemsep=1pt}
\usepackage{xparse}
\RenewDocumentCommand{\paragraph}{s o m}{%
  \par\medskip\noindent\textbf{#3}\quad\ignorespaces
}

\usepackage{algorithm}
\usepackage{algpseudocode}  

\ifdefined\final
  \usepackage[disable]{todonotes}
\else
  \usepackage[textsize=tiny]{todonotes}
\fi

\newcommand{\projectpage}[1]{%
  \begin{center}
    \small
    \vspace{-1.25ex}
    \texttt{Project Page}: \url{#1}
  \end{center}
}

\usepackage{tikz}
\usetikzlibrary{backgrounds}
\usetikzlibrary{arrows.meta}
\definecolor{aBlue}    {HTML}{007AFF}
\definecolor{aIndigo}  {HTML}{5856D6}
\definecolor{aPurple}  {HTML}{AF52DE}
\definecolor{aPink}    {HTML}{FF2D55}
\definecolor{aRed}     {HTML}{FF3B30}
\definecolor{aOrange}  {HTML}{FF9500}
\definecolor{aYellow}  {HTML}{FFCC00}
\definecolor{aGreen}   {HTML}{34C759}
\definecolor{aTeal}    {HTML}{30B0C7}
\definecolor{aGray}    {HTML}{8E8E93}
\definecolor{aGrayLn} {HTML}{C7C7CC}
\definecolor{aBg}      {HTML}{F2F2F7}
\definecolor{aTxt}     {HTML}{1C1C1E}
\definecolor{aSub}     {HTML}{6E6E73}

\title{\bfseries DeepLoop: Depth Scaling for Looped Transformers}

\author{
\textbf{Shuzhen Li}$^1$ \quad
\textbf{Yifan Zhang}$^{1,\dagger}$
\quad \textbf{Jiacheng Guo}$^1$\\[0.5mm]
\textbf{Quanquan Gu}$^2$ \quad
\textbf{Mengdi Wang}$^{1,\dagger}$\\[1.5mm]
$^1$Princeton University \quad $^2$University of California, Los Angeles\\
}
\date{}

\begin{document}
\maketitle

\begingroup
\footnotetext{Corresponding authors: \texttt{yifzhang@princeton.edu}, \texttt{mengdiw@princeton.edu}.}
\endgroup

\begin{abstract}
Looped Transformers scale sequential computation by applying a compact stack of
physical blocks for multiple rounds, increasing unrolled depth without
increasing stored parameters. This reuse changes the residual-scaling problem: 
in an untied Transformer, each residual branch receives and applies its own
parameter update, whereas in a looped Transformer one shared update aggregates
gradients from repeated visits and is read back by those same visits in the
next linearized forward pass. We formalize this tied-depth effect through a
first-order perturbation bound controlled by a visit-alignment coefficient
$\kappa_R$. The bound recovers the DeepNorm exponent when visits decorrelate,
but in the conservative aligned regime it requires the exponent to increase
from $1/4$ to $1/2$ as loop count grows at fixed physical depth. The resulting
method, \textbf{DeepLoop}, keeps the Post-LN DeepNorm architecture with a
normalized branch input and sets
$\alpha=(2N)^{1/2}$ and $\beta=(8N)^{-1/2}$ for unrolled depth $N$. On
GPT-style looped language models at GPT-2 small and GPT-2 medium scale,
DeepLoop is neutral when no physical block is revisited, improves validation
loss at every activated loop count, and improves average downstream accuracy
in all but one of the evaluated settings. These results show that stable recurrent depth requires residual
scaling rules that account for parameter visits, not only nominal layer count.
\end{abstract}

\projectpage{https://github.com/lszshu/DeepLoop}

\begin{figure}[t]
  \centering
\begingroup
\makeatletter
\def\Hv@scale{0.92}%
\makeatother
\renewcommand{\sfdefault}{phv}%
\definecolor{dlBlue}  {HTML}{4C72B0} 
\definecolor{dlOrange}{HTML}{DD8452} 
\definecolor{dlInk}   {HTML}{2B2F36} 
\definecolor{dlMut}   {HTML}{6A7280} 
\definecolor{dlLine}  {HTML}{9AA0A8} 
\definecolor{dlArrow} {HTML}{565D66} 
\begin{tikzpicture}[
  font=\sffamily,
  >={Stealth[length=1.7mm,width=1.4mm]},
  line cap=round, line join=round,
  every node/.style={text=dlInk, outer sep=0pt},
  subS/.style 2 args={rectangle, rounded corners=2pt, draw=#1,
      line width=0.6pt, fill=#2, minimum width=7.6mm, minimum height=6.2mm,
      inner sep=1pt, font=\sffamily\scriptsize},
  subL/.style 2 args={rectangle, rounded corners=2pt, draw=#1,
      line width=0.6pt, fill=#2, minimum width=11.5mm, minimum height=6.8mm,
      inner sep=1pt, font=\sffamily\footnotesize},
  chip/.style={rectangle, rounded corners=2pt, draw=dlLine, fill=black!5,
      line width=0.6pt, minimum height=6.2mm, minimum width=8.5mm,
      inner sep=2.5pt, font=\footnotesize},
  gray box/.style={rectangle, rounded corners=2pt, draw=dlLine, fill=black!4,
      line width=0.6pt, font=\sffamily\footnotesize},
  flow/.style={->, draw=dlArrow, line width=0.6pt,
      shorten >=1.2pt, shorten <=1.2pt},
  head/.style={font=\sffamily\small, anchor=base west, inner sep=0pt},
  mut/.style={font=\sffamily\scriptsize, text=dlMut},
]

\node[head] at (-1.53, 10.15)
  {{\bfseries (a)\; Physical blocks} \textcolor{dlMut}{---\;
   $K{=}2$ shared blocks, stored once}};

\node[rectangle, rounded corners=2pt, draw=dlBlue, fill=dlBlue!6,
      line width=0.8pt, minimum width=38mm, minimum height=16mm]
  (B1) at (0.77, 8.90) {};
\node[font=\sffamily\footnotesize\bfseries, text=dlBlue]
  at (0.77, 9.37) {Block 1\; ($\phi_1$)};
\node[subL={dlBlue}{dlBlue!13}, minimum width=13mm]
  (b1a) at (-0.08, 8.60) {attn};
\node[subL={dlBlue}{dlBlue!30}, minimum width=13mm]
  (b1f) at ( 1.62, 8.60) {ffn};
\draw[flow] (b1a) -- (b1f);

\node[rectangle, rounded corners=2pt, draw=dlOrange, fill=dlOrange!6,
      line width=0.8pt, minimum width=38mm, minimum height=16mm]
  (B2) at (5.37, 8.90) {};
\node[font=\sffamily\footnotesize\bfseries, text=dlOrange]
  at (5.37, 9.37) {Block 2\; ($\phi_2$)};
\node[subL={dlOrange}{dlOrange!13}, minimum width=13mm]
  (b2a) at (4.52, 8.60) {attn};
\node[subL={dlOrange}{dlOrange!30}, minimum width=13mm]
  (b2f) at (6.22, 8.60) {ffn};
\draw[flow] (b2a) -- (b2f);

\draw[flow] (B1) -- (B2);

\draw[flow, rounded corners=3pt]
  (B2.east) -- ++(0.40,0) |- (3.07, 7.55) -| ($(B1.west)+(-0.40,0)$)
  -- (B1.west);
\node[font=\sffamily\footnotesize, fill=white, inner sep=2.5pt]
  at (3.07, 7.55) {loop $\times\,R$ rounds};

\node[head] at (-1.53, 6.45)
  {{\bfseries (b)\; Unrolled execution} \textcolor{dlMut}{---\;
   the same two blocks are revisited in every round}};

\def\yc{5.45}
\node[chip] (xin)  at (-1.105, \yc) {$\mathbf{x}_0$};
\node[chip] (xout) at (14.085, \yc) {$\mathbf{x}_N$};

\foreach \gx/\hue/\pj [count=\i] in
  {0/dlBlue/1, 2.30/dlOrange/2, 4.82/dlBlue/1,
   7.12/dlOrange/2, 9.64/dlBlue/1, 11.94/dlOrange/2}{
  \node[subS={\hue}{\hue!13}] (a\i) at (\gx, \yc) {attn};
  \node[subS={\hue}{\hue!30}] (f\i) at (\gx+1.04, \yc) {ffn};
  \draw[flow] (a\i) -- (f\i);
  \draw[draw=\hue!45, line width=0.5pt, rounded corners=2pt]
    (\gx-0.51, \yc-0.45) rectangle (\gx+1.55, \yc+0.45);
  \node[font=\sffamily\scriptsize, text=\hue] at (\gx+0.52, \yc-0.72)
    {$\phi_{\pj}$};
}

\draw[flow] (xin) -- (a1);
\foreach \i [evaluate=\i as \j using int(\i+1)] in {1,...,5}{
  \draw[flow] (f\i) -- (a\j);}
\draw[flow] (f6) -- (xout);

\draw[dlLine!70, line width=0.45pt, dash pattern=on 2.6pt off 2pt]
  (4.08, \yc+0.5) -- (4.08, \yc-1.25);
\draw[dlLine!70, line width=0.45pt, dash pattern=on 2.6pt off 2pt]
  (8.90, \yc+0.5) -- (8.90, \yc-1.25);
\node[mut] at ( 1.67, \yc-1.05) {round $r{=}1$};
\node[mut] at ( 6.49, \yc-1.05) {round $r{=}2$};
\node[mut] at (11.31, \yc-1.05) {round $r{=}3$};

\node[font=\sffamily\footnotesize, text=dlInk] at (6.50, 3.90)
  {$R{=}3$ rounds $\times$ $K{=}2$ blocks $\;\Rightarrow\;$
   unrolled depth $N{=}KR{=}6$,\quad $M{=}2N{=}12$ sublayer visits};

\node[head] at (-1.53, 3.10) {{\bfseries (c)\; Residual sublayer:}\;\;
  $\mathbf{x}_{i+1} = \mathrm{Norm}\bigl(\alpha\,\mathbf{x}_i
    + f_j(\mathrm{Norm}(\mathbf{x}_i);\phi_j)\bigr)$};

\def\yd{1.95}
\node[chip] (ri) at (-1.105, \yd) {$\mathbf{x}_i$};
\node[circle, draw=dlArrow, fill=white, line width=0.6pt,
      minimum size=4.2mm, inner sep=0pt, font=\footnotesize]
  (add) at (5.05, \yd) {$+$};
\node[gray box, minimum width=13.5mm, minimum height=6.5mm]
  (nrm) at (6.90, \yd) {Norm};
\node[chip] (ro) at (8.80, \yd) {$\mathbf{x}_{i+1}$};
\node[gray box, minimum width=10mm, minimum height=6.5mm]
  (inrm) at (1.15, 0.85) {Norm};
\node[gray box, minimum width=21mm, minimum height=8mm]
  (fj) at (3.10, 0.85) {$f_j(\,\cdot\,;\phi_j)$};

\coordinate (sp) at (0.25, \yd);
\fill[dlArrow] (sp) circle (1.1pt);
\draw[draw=dlArrow, line width=0.6pt] (ri) -- (sp);
\draw[flow] (sp) -- (add);
\node[font=\sffamily\scriptsize, anchor=south] at (2.65, \yd+0.13)
  {{\color{dlMut}skip}\; $\times\,\alpha$};
\draw[flow, rounded corners=3pt] (sp) |- (inrm.west);
\draw[flow] (inrm.east) -- (fj.west);
\draw[flow, rounded corners=3pt] (fj.east) -| (add.south);
\draw[flow] (add) -- (nrm);
\draw[flow] (nrm) -- (ro);

\draw[draw=dlLine, line width=0.45pt, dash pattern=on 1.6pt off 1.4pt]
  (fj.south) -- ++(0, -0.22);
\node[font=\sffamily\scriptsize, anchor=north] at (3.10, 0.20)
  {{\color{dlMut}init gain}\, $\beta$};

\node[head] at (10.63, 3.10) {{\bfseries (d)\; Scaling rule}};
\node[rectangle, rounded corners=2pt, draw=dlInk!75, fill=black!3,
      line width=0.8pt, minimum width=39mm, minimum height=27mm,
      text width=32mm, anchor=north east, align=left, inner sep=8pt,
      font=\sffamily\footnotesize]
  (rule) at (14.53, 2.70) {%
    {\bfseries DeepLoop} \; ($p = \tfrac{1}{2}$)\\[5pt]
    $\alpha = (2N)^{1/2}$\\[2pt]
    $\beta  = (8N)^{-1/2}$\\[5pt]
     };

\end{tikzpicture}%
\endgroup
  \caption{Overview of the DeepLoop framework.
    \textbf{(a) Physical blocks:} A looped Transformer stores $K{=}2$
    physical blocks (each with an attention and an FFN sublayer),
    Block~1 ($\phi_1$) and Block~2 ($\phi_2$).
    \textbf{(b) Unrolled execution:} The whole $K$-block stack is applied
    once per round for $R{=}3$ rounds, giving $N{=}KR{=}6$ unrolled
    blocks and $M{=}2N{=}12$ sublayer visits; the $\phi_j$ label shows
    which physical block is reused.
    \textbf{(c) Residual sublayer:} Each visit applies
    $\mathbf{x}_{i+1}{=}\mathrm{Norm}(\alpha\,\mathbf{x}_i
    + f_j(\mathrm{Norm}(\mathbf{x}_i);\phi_j))$, with the inner
    normalization fixing the branch-input scale, $\alpha$ scaling the skip
    connection, and $\beta$ the per-matrix initialization gain.
    \textbf{(d) Scaling rule:} DeepLoop sets $\alpha{=}(2N)^{1/2}$ and
    $\beta{=}(8N)^{-1/2}$.}
  \label{fig:deeploop-overview}
\end{figure}

\section{Introduction}

Depth is one of the most reliable ways to improve Transformer expressivity, but in standard architectures, depth and parameter count grow together: adding a layer also adds a new set of attention and feed-forward weights
\citep{vaswani2017attention,kaplan2020scaling,hoffmann2022training}. Looped
Transformers decouple these axes by applying a stack of $K$ physical blocks for
$R$ rounds, yielding unrolled depth $N=KR$ while storing only $K$ blocks
(Figure~\ref{fig:deeploop-overview}). This mechanism connects classical depth-wise sharing in Universal Transformers,
ALBERT, and Subformer \citep{dehghani2018universal,lan2019albert,reid2021subformer}
with recent recurrent-depth and test-time-compute models that spend additional
sequential computation without a proportional increase in parameters
\citep{zhu2026scalinglatentreasoninglooped, fu2026discolooploopingdiscreteembeddings, giannou2023looped,yang2023looped,gatmiry2024can,geiping2025scaling,
saunshi2025reasoning}. For the looped depth to become a practical scaling axis, however, its residual parameterization must remain stable as the same physical blocks are revisited many times.

The difficulty is that standard residual-scaling analyses are written for
untied depth. Residual parameterization is decisive for deep Transformer
optimization \citep{xiong2020layer,nguyen2019transformers,liu2020understanding,
huang2020improving,bachlechner2021rezero,wang2024deepnet}. DeepNorm, in
particular, makes very deep Post-LN Transformers trainable by choosing a skip
scale $\alpha$ and a residual-branch initialization gain $\beta$ so that the
first-order effect of parameter updates remains bounded across $M=2N$ residual
sublayer visits \citep{wang2024deepnet}. This calculation assumes that each
unrolled residual sublayer owns a distinct parameter tensor and therefore
contributes one independent update term.

Weight sharing violates precisely this assumption. When a physical sublayer is
visited $R$ times, its optimizer update aggregates the visit-wise gradients
from all rounds. The updated tensor is then reused by all of those visits in
the next linearized forward computation. Tied depth therefore creates two
coupled aggregation paths: the update is written by many visits and then read
by many visits. The size of this effect depends on visit alignment. If
visit-wise gradients and sensitivities are nearly orthogonal across rounds, a
looped model behaves like untied depth up to constants. If they are coherent,
the shared update can acquire an additional factor of $R$.

This paper makes the tied-depth effect explicit. We introduce a
visit-alignment coefficient $\kappa_R$ and show that the first-order stability
condition for a looped Transformer becomes
\[
  M\kappa_R\left(\frac{\beta}{\alpha}\right)^2 = O(1),
\]
rather than the untied DeepNorm condition $M(\beta/\alpha)^2=O(1)$. For the
scaling family $\alpha=(cN)^p$ and $\beta=(dN)^{-p}$, the usual DeepNorm
threshold $p=1/4$ is recovered when visits decorrelate. In the conservative
aligned regime, where $\kappa_R=\Theta(R)$ and $K$ is fixed while $R$ grows, the threshold becomes $p=1/2$.

The resulting method is \textbf{DeepLoop}: keep the DeepNorm Post-LN
architecture, augmented with the branch-input normalization that is standard
in recurrent-depth models \citep{geiping2025scaling} and provably inert to
the perturbation analysis, and use the loop-aware scaling rule
\[
  \alpha=(2N)^{1/2},
  \qquad
  \beta=(8N)^{-1/2}.
\]
DeepLoop introduces no gates, learned residual coefficients, auxiliary losses,
or architecture-specific tuning constants. It is a deterministic residual
parameterization for the regime where effective depth is increased by
revisiting shared blocks.

We evaluate DeepLoop in controlled GPT-style looped language modeling
ablations at GPT-2 small and GPT-2 medium scales. DeepLoop is essentially
neutral at $R=1$, where no physical block is revisited, and consistently
improves final validation loss at larger loop counts. The advantage transfers
to an eight-task language-model evaluation suite: downstream averages are
close at $R=1$ and, with a single exception at one intermediate loop count,
move in favor of DeepLoop as recurrent depth grows. These
results support the central prediction of the analysis: residual scaling
should depend on how depth is realized, not only on the nominal number of
unrolled layers.

Our contributions are:
\begin{enumerate}[leftmargin=*, itemsep=1pt, topsep=1pt]
    \item We identify the tied-depth aggregation mechanism that is absent from
    untied residual-scaling analyses: a shared update is accumulated across
    repeated visits and then read through those same visits;
    \item We derive a loop-aware first-order perturbation bound with an
    explicit visit-alignment coefficient $\kappa_R$, recovering DeepNorm in the
    decorrelated regime and yielding a $p=1/2$ exponent threshold in the
    aligned fixed-physical-depth regime;
    \item We propose DeepLoop, the scaling rule
    $\alpha=(2N)^{1/2}$ and $\beta=(8N)^{-1/2}$, as a one-line conservative
    parameterization for Post-LN looped Transformers with a normalized branch
    input;
    \item We provide controlled pre-training and downstream ablations at GPT-2
    small and GPT-2 medium scale, showing that the parameterization becomes
    useful precisely when the loop count is greater than one, with the
    exponent probed separately by a $p$-sweep (Appendix~\ref{app:psweep}).
\end{enumerate}
\section{Background}
\label{sec:background}

\subsection{Looped Transformers and effective depth}
\label{sec:background:looped}

A standard depth-$N$ Transformer applies $N$ distinct blocks once. A
looped Transformer instead chooses $K$ physical blocks and applies them
for $R$ rounds,
\begin{equation}
  \mathbf{x}_{r,k+1}
  = B_k(\mathbf{x}_{r,k};\phi_k),
  \qquad
  k=1,\ldots,K,\quad r=1,\ldots,R,
  \qquad
  \mathbf{x}_{r+1,1}=\mathbf{x}_{r,K+1},
\end{equation}
so the effective depth is $N=KR$ while the block parameters
$\phi_1,\ldots,\phi_K$ are stored only once. The Universal Transformer
\citep{dehghani2018universal} corresponds to the fully recurrent
extreme in which the same transition is reused across depth; ALBERT
\citep{lan2019albert} similarly ties Transformer parameters across
layers. The regime studied here keeps $K$ fixed or small and increases
$R$, thereby increasing test-time compute and unrolled depth without
increasing the number of physical blocks.

\subsection{Residual normalization}
\label{sec:background:residual}

Let $g_\ell$ denote one residual sublayer, either attention or an MLP.
The two common normalization placements are
\begin{align*}
  \text{Pre-LN:}\quad
  \mathbf{x}_{\ell+1}
  = \mathbf{x}_\ell + g_\ell(\mathrm{Norm}(\mathbf{x}_\ell)),
  \qquad
  \text{Post-LN:}\quad
  \mathbf{x}_{\ell+1}
  = \mathrm{Norm}\left(\mathbf{x}_\ell + g_\ell(\mathbf{x}_\ell)\right).
\end{align*}
Pre-LN improves optimization stability in deep Transformers, while
Post-LN can preserve a more expressive residual stream
\citep{xiong2020layer,nguyen2019transformers}. Recent recurrent-depth models
additionally normalize the branch input inside a post-normalized residual, a
sandwich placement that combines the two schemes \citep{geiping2025scaling};
Section~\ref{sec:method:setup} adopts this placement and shows that it leaves
the perturbation analysis unchanged. DeepNorm
\citep{wang2024deepnet} modifies the Post-LN residual path by scaling
the skip connection before normalization:
\begin{align}
  \mathbf{x}_{\ell+1}
  = \mathrm{Norm}\left(\alpha\,\mathbf{x}_\ell
  + g_\ell(\mathbf{x}_\ell;\theta_\ell)\right).
\end{align}
For an encoder-only or decoder-only Transformer with $N$ blocks and
$M=2N$ residual sublayer applications, DeepNorm sets
\begin{equation}
  \alpha=(2N)^{1/4},
  \qquad
  \beta=(8N)^{-1/4},
  \label{eq:deepnorm-scales-background}
\end{equation}
Here $\beta$ is an initialization gain, not an additional runtime multiplier on
the residual stream. If $\mathcal{S}^{\mathrm{DN}}_\ell$ denotes the
residual-branch matrices scaled by DeepNet in sublayer $\ell$, such as the value
and output projections in attention and the feed-forward matrices, then
DeepNorm initializes
\begin{equation}
  W_{\ell,q}^{(0)}
  =
  \beta_{\mathrm{DN}}\,
  \widetilde W_{\ell,q}^{(0)},
  \qquad
  q\in\mathcal{S}^{\mathrm{DN}}_\ell,
  \qquad
  \beta_{\mathrm{DN}}=(8N)^{-1/4},
  \label{eq:deepnorm-beta-usage}
\end{equation}
with the usual unscaled initializer used for
$\widetilde W_{\ell,q}^{(0)}$. The quantity that enters the perturbation
argument is therefore
\begin{equation}
  \frac{\beta_{\mathrm{DN}}}{\alpha_{\mathrm{DN}}}
  =
  \frac{(8N)^{-1/4}}{(2N)^{1/4}}
  =
  \frac{1}{2\sqrt{N}}.
  \label{eq:deepnorm-beta-over-alpha-background}
\end{equation}
The useful way to summarize the DeepNorm calculation is not a bound on
$\alpha^2\beta$, but the first-order update condition
\begin{equation}
  M\left(\frac{\beta}{\alpha}\right)^2 = \Theta(1).
  \label{eq:deepnorm-condition-background}
\end{equation}
Indeed, substituting Eq.~(\ref{eq:deepnorm-beta-over-alpha-background}) gives
$2N(\beta/\alpha)^2=1/2$.

\subsection{Why weight sharing changes the depth scaling}
\label{sec:background:sharing}

In a non-shared depth-$N$ Transformer, each sublayer parameter receives
one gradient contribution per input. In a looped Transformer, the same
physical parameter is visited once per round, so its gradient is a sum
over $R$ visits. If the visit-wise contributions decorrelate, the loop
behaves like untied depth up to constants. If the visits are aligned,
the tied update can be $R$ times larger, and that same update is then
read by all $R$ visits in the unrolled computation. DeepLoop is the
corresponding conservative correction to Eq.~(\ref{eq:deepnorm-condition-background}): it chooses a smaller
update-to-residual ratio $\beta/\alpha$ so that the first-order effect of the tied parameter update remains bounded after all $R$ visits are unrolled.

\section{DeepLoop Transformer}
\label{sec:method}

A looped Transformer reuses the same $K$ physical blocks for $R$ rounds. Its
unrolled depth is $N=KR$, and the number of residual-sublayer visits is
$M=2N$. DeepLoop keeps the DeepNorm architecture with a normalized branch
input and uses the exponent
$p=1/2$ rather than $p=1/4$ in the fixed-physical-depth, increasing-loop-count
regime. The derivation separates two facts: RMSNorm exposes the residual branch
through a factor $1/\alpha$, and parameter tying changes the first-order update
sum from $M$ terms to $M\kappa_R$ terms, where $\kappa_R$ measures alignment
across loop visits.

\subsection{Setup: looped post-normalized block}
\label{sec:method:setup}

Each physical block contains two residual sublayers: attention and MLP. Let
$j=(k,s)$ index a physical sublayer, with $k\in\{1,\ldots,K\}$ and
$s\in\{\mathrm{attn},\mathrm{ffn}\}$, and let $i=(r,j)$ denote its $r$-th
unrolled visit. We write $J=2K$ for the number of physical residual sublayers
and $M=JR=2KR=2N$ for the number of unrolled visits. DeepLoop uses the
post-normalized sandwich block
\begin{align}
  \mathbf{x}_{i+1}
  = \mathrm{Norm}\left(
      \alpha\,\mathbf{x}_i
      + f_j(\mathrm{Norm}(\mathbf{x}_i);\phi_j)
    \right),
  \qquad i=1,\ldots,M,
  \label{eq:deeploop-block}
\end{align}
where the same physical parameter $\phi_j$ is reused for all $r=1,\ldots,R$.
In our implementation, both normalizations are RMSNorm. The outer
normalization restores the residual-stream scale after every visit; the
inner normalization fixes the scale of the branch input. The loop entry is
normalized by an input-embedding RMSNorm, and every later state is an output
of the outer normalization, so at initialization, where the perturbation
analysis is conducted and all normalization gains equal one, the stream
entering each visit satisfies $\mathrm{RMS}(\mathbf{x}_i)=1$. On this
unit-RMS stream the inner normalization acts as the identity, and
Lemma~\ref{lem:inner-norm} in Appendix~\ref{app:proofs} shows that it does
not enlarge the visit-wise sensitivity factors: the analysis below applies
verbatim to Eq.~(\ref{eq:deeploop-block}) with or without the inner
normalization. Its role is architectural rather than analytical: by
construction it presents every visit with a unit-RMS branch input, whatever
the residual-stream scale does during training, which is the branch-input
regime under which the visit-wise constants of
Assumption~\ref{ass:local-sensitivity} below are calibrated. The gain $\beta$ is a
per-matrix initialization gain applied to the residual-branch matrices specified
by DeepNorm. It should not be read as the standard deviation of the whole
sublayer output, since a branch may contain multiple scaled matrices. Concretely,
for each physical sublayer $j$, let $\mathcal{S}_j$ be the set of matrices that
receive the DeepNorm gain. DeepLoop initializes
\begin{equation}
  W_{j,q}^{(0)}
  =
  \beta_{\mathrm{DL}}\,
  \widetilde W_{j,q}^{(0)},
  \qquad
  q\in\mathcal{S}_j,
  \qquad
  \beta_{\mathrm{DL}}=(8N)^{-1/2}.
  \label{eq:deeploop-beta-usage}
\end{equation}
The forward recurrence of Eq.~(\ref{eq:deeploop-block}) then uses these
initialized parameters at every visit; $\beta$ is not re-applied as a
separate multiplicative factor on each visit. With
$\alpha_{\mathrm{DL}}=(2N)^{1/2}$, the DeepLoop update-to-residual ratio is
\begin{equation}
  \frac{\beta_{\mathrm{DL}}}{\alpha_{\mathrm{DL}}}
  =
  \frac{(8N)^{-1/2}}{(2N)^{1/2}}
  =
  \frac{1}{4N}.
  \label{eq:deeploop-beta-over-alpha-setup}
\end{equation}

\begin{assumption}[DeepNorm-scale local sensitivity]
\label{ass:local-sensitivity}
For each unrolled visit $i=(r,j)$, the norm of the linearized output map from a
residual-branch parameter perturbation at that visit and the norm of the corresponding visit-wise effective update are both $O(\beta/\alpha)$.
\end{assumption}

Assumption~\ref{ass:local-sensitivity} is the local scaling condition used in
the DeepNet perturbation argument. Constants depending on width, normalization
gain, attention heads, learning rate, or optimizer preconditioning are absorbed
into the $O(\cdot)$ notation; by Lemma~\ref{lem:inner-norm}, on the unit-RMS
stream of the initialized model the inner
normalization of Eq.~(\ref{eq:deeploop-block}) contributes to these constants
only factors of operator norm one. The looped analysis below changes
only how the visit-wise terms are aggregated when parameters are tied.

\subsection{Forward stability}
\label{sec:method:fwd}

RMSNorm restores unit RMS at every sublayer, so the forward signal scale is
not the binding constraint; the relevant constraint is the first-order
sensitivity of the final output to a residual-branch parameter update. A
direct expansion of $\mathrm{RMSNorm}(\alpha\mathbf{x}+\mathbf{z})$ around
$\mathrm{RMS}(\mathbf{z})/\alpha\to 0$ shows that the residual branch enters
the normalized direction through a factor of $1/\alpha$, which is why the
relevant DeepNorm ratio is $\beta/\alpha$ rather than $\beta$ alone. The
inner normalization of Eq.~(\ref{eq:deeploop-block}) complements this
mechanism on the input side: it pins the branch input to unit RMS at every
visit, so the branch-output scale is governed by the $\beta$-scaled matrices
alone. The
formal statement and proof are
Lemma~\ref{lem:rmsnorm-local} in Appendix~\ref{app:proofs}.

\subsection{Depth-untied DeepNorm bound}
\label{sec:method:bwd}

The DeepNet perturbation argument bounds the first-order output change
$\Delta F=F(\mathbf{x};\theta+\delta\theta)-F(\mathbf{x};\theta)$ by summing
visit-wise sensitivities. Here \emph{depth-untied} means that residual-sublayer
parameters are not shared across unrolled depth. This is orthogonal to the
standard tying of input and output token embeddings: tied embeddings, when used,
are outside the $M=2N$ residual-sublayer count and are not among the
residual-branch matrices scaled by $\beta$. Under
Assumption~\ref{ass:local-sensitivity}, each visit contributes an output
sensitivity $O(\beta/\alpha)$ and an effective update $O(\beta/\alpha)$, so
summing $M=2N$ depth-untied residual-sublayer visits gives
\begin{equation}
  \|\Delta F\|
  \le
  C'\,M\left(\frac{\beta}{\alpha}\right)^2,
  \label{eq:update-bound}
\end{equation}
and a sufficient first-order stability condition is
\begin{equation}
  M\left(\frac{\beta}{\alpha}\right)^2 = O(1).
  \label{eq:deepnet-condition}
\end{equation}
The DeepNorm choice $\alpha=(2N)^{1/4}$ and $\beta=(8N)^{-1/4}$ satisfies
$M(\beta/\alpha)^2 = 1/2$. The formal proposition statement and proof are
Proposition~\ref{prop:untied} in Appendix~\ref{app:proofs}.

\subsection{Loop correction: aligned visits require \texorpdfstring{$p=1/2$}{p=1/2}}
\label{sec:method:loop}

Depth-wise residual-sublayer weight sharing changes Eq.~(\ref{eq:update-bound}).
For a physical sublayer $j$, let $G_{r,j}$ be the visit-wise effective update
that would be applied if visit $(r,j)$ had its own depth-untied residual-branch
parameter, and let $U_{r,j}$ be the corresponding linearized
output-sensitivity operator. With tied residual-sublayer parameters, the
update to $\phi_j$ is proportional to $\sum_{r=1}^{R}G_{r,j}$, and the output
perturbation reads this same update through all visits. The first-order tied
perturbation therefore has the schematic double-sum form
\begin{equation}
  \Delta F_{\mathrm{tied}}
  = -\eta \sum_{j=1}^{J}
  \left(\sum_{r=1}^{R} U_{r,j}\right)
  \left(\sum_{t=1}^{R} G_{t,j}\right) + O(\eta^2),
  \label{eq:shared-grad}
\end{equation}
where tensor contractions and optimizer-dependent constants are absorbed into
$U_{r,j}$ and $G_{r,j}$.

Define the visit-alignment coefficient
\begin{equation}
  \kappa_R
  :=
  \max_{j\in\{1,\ldots,J\}}
  \frac{
    \left\|\sum_{r=1}^{R} U_{r,j}\right\|
    \left\|\sum_{r=1}^{R} G_{r,j}\right\|
  }{
    R C_U C_G(\beta/\alpha)^2
  },
  \label{eq:kappa-def}
\end{equation}
where $\|U_{r,j}\|\le C_U\beta/\alpha$ and
$\|G_{r,j}\|\le C_G\beta/\alpha$. The triangle inequality gives
$0\le\kappa_R\le R$. Independent or nearly orthogonal visits give
$\kappa_R=O(1)$, while fully aligned visits give $\kappa_R=\Theta(R)$.

Substituting the definition of $\kappa_R$ from Eq.~(\ref{eq:kappa-def}) into
Eq.~(\ref{eq:shared-grad}) and applying submultiplicativity gives the looped
version of the DeepNorm bound,
\begin{equation}
  \|\Delta F\|
  \le
  C''\,M\kappa_R\left(\frac{\beta}{\alpha}\right)^2,
  \label{eq:loop-update-bound}
\end{equation}
so a sufficient tied-depth stability condition is
\begin{equation}
  M\kappa_R\left(\frac{\beta}{\alpha}\right)^2 = O(1).
  \label{eq:loop-condition}
\end{equation}
In the worst-case aligned regime $\kappa_R=\Theta(R)$, this reduces to
$MR(\beta/\alpha)^2=O(1)$. The formal proposition and the aligned-case
corollary are Proposition~\ref{prop:tied-perturbation} and
Corollary~\ref{cor:aligned-loop} in Appendix~\ref{app:proofs}.

Now consider the family $\alpha=(cN)^p$, $\beta=(dN)^{-p}$ with constants
$c,d>0$. Its update-to-residual ratio is
\begin{equation}
  \frac{\beta}{\alpha}
  =
  (cd)^{-p}N^{-2p}.
  \label{eq:beta-over-alpha}
\end{equation}

\begin{proposition}[Exponent threshold]
\label{prop:exponent-threshold}
Assume $K$ is fixed, $N=KR$, $M=2N$, and
$\kappa_R=\Theta(R^\gamma)$ for some $\gamma\in[0,1]$. For the scaling family
$\alpha=(cN)^p$, $\beta=(dN)^{-p}$, the tied-depth condition
\[
  M\kappa_R(\beta/\alpha)^2=O(1)
\]
holds uniformly as $R\to\infty$ if and only if
\[
  p\ge \frac{1+\gamma}{4}.
\]
\end{proposition}

Proposition~\ref{prop:exponent-threshold} shows where the new exponent comes
from. When visits decorrelate, $\gamma=0$, and the DeepNorm threshold $p=1/4$
is recovered. In the conservative aligned case, $\gamma=1$, so the threshold is
$p=1/2$. Equivalently,
\begin{equation}
  M R\left(\frac{\beta}{\alpha}\right)^2
  =
  2NR\,(cd)^{-2p}N^{-4p}.
  \label{eq:p-family-loop}
\end{equation}
At fixed $K$, $R=N/K$, so Eq.~(\ref{eq:p-family-loop}) is bounded only when
$p\ge 1/2$. DeepNorm's $p=1/4$ leaves a residual growth of $\Theta(R)$ in the
aligned shared-loop bound. The DeepLoop choice is
\begin{align*}
  \alpha=(2N)^{1/2},
  \qquad
  \beta=(8N)^{-1/2}.
\end{align*}
We keep the DeepNorm constants $(c,d)=(2,8)$ for a two-sublayer decoder block
and change only the exponent. Equivalently, DeepLoop uses the same per-matrix
definition as Eq.~(\ref{eq:deeploop-beta-usage}) with
$\beta=\beta_{\mathrm{DL}}$. With this choice,
\[
  \frac{\beta}{\alpha}
  =
  \frac{1}{4N},
  \qquad
  MR\left(\frac{\beta}{\alpha}\right)^2
  =
  2NR\cdot \frac{1}{16N^2}
  =
  \frac{1}{8K},
\]
so the worst-case aligned tied-depth bound remains $O(1)$ at fixed physical
depth $K$. Proofs of Proposition~\ref{prop:tied-perturbation} and
Proposition~\ref{prop:exponent-threshold} are in Appendix~\ref{app:proofs}.
Appendix~\ref{app:psweep} sweeps $p$ over a single grid at $R{=}3$ and
finds an empirical training-stability transition that brackets this
$p{=}1/2$ prediction.

\subsection{Summary and comparison to DeepNorm}
\label{sec:method:summary}

\begin{center}
\begin{tabular}{lcc}
  \toprule
  & DeepNorm & DeepLoop \\
  \midrule
  Residual scale $\alpha$
    & $(2N)^{1/4}$ & $(2N)^{1/2}$ \\
  Init gain $\beta$
    & $(8N)^{-1/4}$ & $(8N)^{-1/2}$ \\
  Per-matrix use of $\beta$
    & $W^{(0)}\leftarrow\beta\,\widetilde W^{(0)}$
    & $W^{(0)}\leftarrow\beta\,\widetilde W^{(0)}$ \\
  Update-to-residual $\beta/\alpha$
    & $1/(2\sqrt{N})$ & $1/(4N)$ \\
  Untied bound $M(\beta/\alpha)^2$
    & $\Theta(1)$ & $\Theta(N^{-1})$ \\
  Aligned tied-loop bound $MR(\beta/\alpha)^2$ at fixed $K$
    & $\Theta(R)$ & $\Theta(1/K)$ \\
  \bottomrule
\end{tabular}
\end{center}

DeepLoop is therefore a one-line conservative correction for tied depth:
increase the exponent from $p=1/4$ to $p=1/2$. It introduces no gates,
learnable residual scalars, or additional tuning constants; the correction
itself touches only the exponent. The two normalizations per sublayer in
Eq.~(\ref{eq:deeploop-block}) are the standard RMSNorms already used by
recurrent-depth models \citep{geiping2025scaling}, and
Lemma~\ref{lem:inner-norm} shows the inner one is inert to the analysis. The correction is not a claim that every looped model empirically attains worst-case alignment. Rather, it is the minimal exponent for uniform boundedness under aligned visit-wise contributions at fixed physical depth.

\section{Application to Hierarchical Recurrent Reasoners}
\label{sec:hrm}

The DeepLoop derivation in Section~\ref{sec:method} treats a single looped
block visited for $R$ rounds. Recent recurrent-reasoning architectures combine
looping with two further structural choices: a hierarchy of recurrent modules
executed at different rates, and training-time gradient truncation across the
outer recurrence \citep{wang2025hrm,geiping2025scaling}. The Hierarchical
Reasoning Model \citep{wang2025hrm} instantiates both: a high-level module
$\mathcal{H}$ and a low-level module $\mathcal{L}$ are unrolled jointly, and 
gradients are computed through only the last outer cycle. This section shows
that the framework of Section~\ref{sec:method} extends to this regime by
replacing the unrolled visit count $M$ with a gradient-visible visit count
$M_{\mathrm{g}}$ and tracking per-module visit counts separately. We do not
introduce new optimization machinery or a new variant of DeepLoop; we describe
how the existing analysis specializes when training-time depth and forward
depth differ.

\subsection{Setup: two-module hierarchical loops}
\label{sec:hrm:setup}

Let $\mathcal{H}$ contain $K_H$ physical blocks (each with two residual
sublayers, attention and MLP) and let $\mathcal{L}$ contain $K_L$ physical
blocks of the same shape. The hierarchical blocks use the post-normalized
recurrence in the form of Eq.~(\ref{eq:deeploop-block}) without the inner
branch normalization, which is the form for which the propositions of
Appendix~\ref{app:proofs} are stated directly, so the analysis applies
unchanged.
The hierarchical recurrence runs $C$ outer cycles;
within each outer cycle, $\mathcal{L}$ iterates $C_L$ times before a single
update of $\mathcal{H}$:
\[
  \text{cycle } c:\quad
  \mathbf{z}_L \leftarrow \mathcal{L}^{(C_L)}(\mathbf{z}_L,\mathbf{z}_H),\qquad
  \mathbf{z}_H \leftarrow \mathcal{H}(\mathbf{z}_H,\mathbf{z}_L),
  \quad c=1,\ldots,C.
\]
Counting attention and MLP residual sublayers, the total unrolled visit count
is
\begin{equation}
  M
  = 2\,C\,(K_H + C_L K_L),
  \label{eq:hrm-forward-M}
\end{equation}
which generalizes the single-module count $M=2KR$ used in
Section~\ref{sec:method}. The set of distinct physical residual sublayers has
size $J = J_H + J_L$ with $J_H = 2K_H$ and $J_L = 2K_L$. Each visit
$i=(c,\ell,j)$ has an outer-cycle index $c$, an inner $\mathcal{L}$-iterate
index $\ell$ (only meaningful for $j\in\mathcal{L}$), and a physical-sublayer
index $j$.

\subsection{Loop-aware bound under one-step gradient approximation}
\label{sec:hrm:gradtrunc}

Hierarchical recurrent reasoners are typically trained with a one-step
gradient approximation: the forward pass uses all $C$ outer cycles, but the
backward pass is computed against the last cycle only, with $\mathbf{z}_H$
and $\mathbf{z}_L$ from earlier cycles detached from the autograd graph. This truncation does not alter the numerical forward values, so Lemma~\ref{lem:rmsnorm-local} applies unchanged locally. The perturbation bound below is for the truncated training graph: the states entering the last outer cycle are treated as fixed inputs, and both the sensitivity terms and the update terms from earlier cycles are absent. A perturbation bound for the full untruncated forward map after an optimizer step would require an additional forward-visible read factor.

Define the gradient-visible visit count
\begin{equation}
  M_{\mathrm{g}}
  := 2\,(K_H + C_L K_L),
  \label{eq:hrm-Mg}
\end{equation}
so $M = C\,M_{\mathrm{g}}$. The per-physical-sublayer round counts on the
gradient-visible cycle are
\begin{equation}
  R_{\mathrm{g}}^{(\mathcal{H})} = 1,
  \qquad
  R_{\mathrm{g}}^{(\mathcal{L})} = C_L.
  \label{eq:hrm-per-module-R}
\end{equation}
Repeating the argument leading to Eq.~(\ref{eq:loop-update-bound}) but
restricting the visit-index sums to the gradient-visible cycle gives the
two-module loop-aware bound
\begin{equation}
  \|\Delta F\|
  \le
  C''\!
  \left[
    J_H\,R_{\mathrm{g}}^{(\mathcal{H})}\,\kappa_{\mathrm{g}}^{(\mathcal{H})}
    +
    J_L\,R_{\mathrm{g}}^{(\mathcal{L})}\,\kappa_{\mathrm{g}}^{(\mathcal{L})}
  \right]
  \left(\frac{\beta}{\alpha}\right)^2,
  \label{eq:hrm-per-module-bound}
\end{equation}
where $\kappa_{\mathrm{g}}^{(\mathcal{H})},\kappa_{\mathrm{g}}^{(\mathcal{L})}$
are the visit-alignment coefficients of Eq.~(\ref{eq:kappa-def}) restricted
to each module's gradient-visible visits. Two structural facts simplify
Eq.~(\ref{eq:hrm-per-module-bound}). First,
$\kappa_{\mathrm{g}}^{(\mathcal{H})}\le 1$ trivially: with
$R_{\mathrm{g}}^{(\mathcal{H})}=1$, each physical $\mathcal{H}$-sublayer
contributes a single-term inner sum, and the alignment coefficient collapses
to a constant. The truncation thus places $\mathcal{H}$ in the
untied-DeepNorm regime even though $\mathcal{H}$ is forward-visited $C$
times. Second, $\kappa_{\mathrm{g}}^{(\mathcal{L})}\in[0,C_L]$ by the
triangle inequality applied to the inner $\mathcal{L}$-iterate sum.
Substituting these facts, upper-bounding $\kappa_{\mathrm{g}}^{(\mathcal{H})}$ by one, and applying the per-module counts in
Eq.~(\ref{eq:hrm-per-module-R}) yields the sufficient stability condition
\begin{equation}
  M_{\mathrm{g}}\,\bar\kappa_{\mathrm{g}}\,
  \left(\frac{\beta}{\alpha}\right)^2
  = O(1),
  \qquad
  \bar\kappa_{\mathrm{g}}
  :=
  \frac{J_H + J_L\,C_L\,\kappa_{\mathrm{g}}^{(\mathcal{L})}}{M_{\mathrm{g}}},
  \label{eq:hrm-cond}
\end{equation}
which has the same structural form as Eq.~(\ref{eq:loop-condition}) with $M$
replaced by $M_{\mathrm{g}}$ and $\kappa_R$ replaced by the gradient-visible
aggregate $\bar\kappa_{\mathrm{g}}\in[\,J_H/M_{\mathrm{g}},\,1+J_L C_L(C_L-1)/M_{\mathrm{g}}\,]$.

The decomposition (\ref{eq:hrm-per-module-bound}) also licenses a per-module
scaling family. Each summand depends only on the $(\alpha,\beta)$ used at
the corresponding module, so a hierarchical reasoner may be parameterized
with $(\alpha_{\mathcal{H}},\beta_{\mathcal{H}})$ and
$(\alpha_{\mathcal{L}},\beta_{\mathcal{L}})$ chosen independently. The
admissible exponent at each module is determined by the corresponding
summand of Eq.~(\ref{eq:hrm-per-module-bound}), not by the global visit
count.

\subsection{Predicted exponent regime}
\label{sec:hrm:exp}

Let $N_{\mathrm{g}}:=M_{\mathrm{g}}/2$ denote the gradient-visible
block-equivalent depth. Combining Eq.~(\ref{eq:hrm-cond}) with the same
two-sublayer scaling family used above, $\alpha=(c N_{\mathrm{g}})^p$, $\beta=(d N_{\mathrm{g}})^{-p}$, gives, up to constant factors,
\begin{equation}
  M_{\mathrm{g}}\,\bar\kappa_{\mathrm{g}}\,
  (cd)^{-2p}\,N_{\mathrm{g}}^{-4p}
  =
  \Theta\!\left((cd)^{-2p}\,\bar\kappa_{\mathrm{g}}\,M_{\mathrm{g}}^{1-4p}\right),
  \label{eq:hrm-pthreshold}
\end{equation}
and the asymptotic threshold for boundedness depends on how
$\bar\kappa_{\mathrm{g}}$ and $M_{\mathrm{g}}$ co-scale. We record the two
limit cases that mirror Proposition~\ref{prop:exponent-threshold}.

\paragraph{Decorrelated inner cycles
($\kappa_{\mathrm{g}}^{(\mathcal{L})}=O(1)$).}
Then $\bar\kappa_{\mathrm{g}}=O(1)$ and Eq.~(\ref{eq:hrm-pthreshold}) is
bounded as $M_{\mathrm{g}}\to\infty$ if and only if
$  p \ge 1/4.$
Both modules sit in the untied-DeepNorm regime: the $\mathcal{H}$ module by
truncation, and the $\mathcal{L}$ module by the inner-cycle decorrelation.

\paragraph{Aligned inner cycles
($\kappa_{\mathrm{g}}^{(\mathcal{L})}=\Theta(C_L)$).}
Then the $\mathcal{L}$ contribution to $\bar\kappa_{\mathrm{g}}$ is $\Theta(C_L^2\,J_L/M_{\mathrm{g}})$. Two
sub-cases follow from how $M_{\mathrm{g}}$ is grown.
\emph{(i)} If $C_L$ is fixed and $K_H,K_L\to\infty$, then
$\bar\kappa_{\mathrm{g}}=\Theta(1)$ and the threshold is again
$p\ge 1/4$.
\emph{(ii)} If $K_H,K_L$ are fixed and $C_L\to\infty$, then
$M_{\mathrm{g}}=\Theta(C_L)$ and
$\bar\kappa_{\mathrm{g}}=\Theta(C_L)$, so
Eq.~(\ref{eq:hrm-pthreshold}) is bounded if and only if
$  p \ge 1/2,$
which recovers the fixed-physical-depth tied-loop threshold of
Proposition~\ref{prop:exponent-threshold} along the inner-cycle axis.

For a hierarchical recurrent reasoner as actually built, these two limits are
not equally likely possibilities but a binding constraint and a slack one. The
once-per-step $\mathcal{H}$ module is truncated by the one-step gradient,
contributes a single update term, and therefore requires only $p\ge 1/4$. The
$\mathcal{L}$ module is the one that realizes the model's effective depth: it
revisits its shared blocks at fixed physical depth, the $C_L$ axis of
sub-case~(ii), and weight tying is adopted precisely so that those revisits
implement the \emph{same} operation, which aligns their visit-wise gradients
and sensitivities ($\kappa_{\mathrm{g}}^{(\mathcal{L})}=\Theta(C_L)$). The
$\mathcal{L}$ module therefore requires $p\ge 1/2$. A single shared residual
exponent must satisfy the more demanding module, so the framework predicts
\begin{equation}
  p = 1/2
  \label{eq:hrm-p-prediction}
\end{equation}
for a hierarchical recurrent reasoner, the same loop-aware exponent as the
single-module backbone of Section~\ref{sec:method}. This identifies the
binding regime for the architecture as trained, in the same worst-case sense
as Section~\ref{sec:method}: we do not claim that the alignment
$\kappa_{\mathrm{g}}^{(\mathcal{L})}=\Theta(C_L)$ is empirically attained,
but the recurrent depth that makes the model work is grown along the aligned
inner-cycle axis, exactly the regime whose worst-case threshold is $1/2$, and
$p=1/2$ simultaneously satisfies the
once-per-step module (for which $1/2\ge 1/4$). Section~\ref{sec:experiments} confirms the prediction empirically on ARC-AGI.

\paragraph{Forward signal at the loop entry.}
A practical detail of hierarchical reasoners is the addition of a learned
task or puzzle embedding $\mathbf{e}_{\mathrm{task}}$ to the token embedding
$\mathbf{e}_{\mathrm{tok}}$ before the loop. The choice of whether to apply
RMSNorm to $\mathbf{e}_{\mathrm{tok}}+\mathbf{e}_{\mathrm{task}}$ at the loop
entry affects only the forward signal scale $\mathrm{RMS}(\mathbf{x}_0)$ in
Lemma~\ref{lem:rmsnorm-local}; it does not enter
Eq.~(\ref{eq:hrm-per-module-bound}) or the threshold derivation
(\ref{eq:hrm-pthreshold}). Input-side normalization and the residual scaling rule are therefore separable design choices within this framework, even
though both can independently influence the relative magnitude of task-conditioning information across repeated visits.

\subsection{Summary}
\label{sec:hrm:summary}

The DeepLoop perturbation argument extends to hierarchical recurrent
reasoners by two substitutions. When training uses a one-step gradient
approximation, the visit count entering the bound is the gradient-visible
$M_{\mathrm{g}}$ from Eq.~(\ref{eq:hrm-Mg}) rather than the forward $M$ from
Eq.~(\ref{eq:hrm-forward-M}). When the architecture splits the recurrence
across modules with different per-step visit counts, the bound decomposes
into per-module summands and admits asymmetric per-module
$(\alpha,\beta)$ assignments. Because a hierarchical recurrent reasoner grows
its effective depth by aligned revisits of the shared inner module at fixed
physical depth, the binding threshold from Eq.~(\ref{eq:hrm-pthreshold}) is
$p=1/2$, the same loop-aware exponent as the single-module backbone.
Section~\ref{sec:experiments} confirms this on ARC-AGI: the DeepLoop exponent
$p=1/2$ improves voted accuracy over the vanilla HRM baseline across the full
evaluation ladder.

\section{Experiments}
\label{sec:experiments}

\subsection{Validation loss}
\label{sec:experiments:valloss}

We compare DeepLoop against a baseline that already incorporates the same
looped residual-block sharing, tied input-output embeddings, and an
input-embedding RMSNorm. The baseline (\texttt{base}) is the standard
pre-normalized looped Transformer: each sublayer applies
$\mathbf{x}\leftarrow\mathbf{x}+g(\mathrm{Norm}(\mathbf{x}))$ with standard
initialization and no residual scaling. The DeepLoop variant instead uses the
post-normalized sandwich block of Eq.~(\ref{eq:deeploop-block}) with the
scaling rule enabled on every block. The comparison is therefore between the
modern pre-LN default and the complete DeepLoop parameterization; the
stability effect of the exponent alone, within the fixed sandwich block of
Eq.~(\ref{eq:deeploop-block}), is probed by the
$p$-sweep of Appendix~\ref{app:psweep}. Both configurations share the same
attention and MLP backbone, optimizer, data pipeline, and random seed. The
loop count $R \in \{1,3,5,7\}$ moves between rows of
Table~\ref{tab:loop-valloss}, and the two method rows differ exactly in the
block parameterization: the pre-LN block with standard initialization versus
Eq.~(\ref{eq:deeploop-block}) with the DeepLoop scaling rule.

The GPT-2 small runs use a GPT-MHA-RoPE model trained on
FineWeb-Edu for 50B tokens (100K optimizer steps, context length 1024,
global batch size 480) on 4$\times$H200~141GB GPUs. The GPT-2 medium
runs use the same backbone with hidden size enlarged from $768$ to
$1024$ and the number of layers doubled from $12$ to $24$, trained on
the same data and schedule on 8$\times$H200~141GB GPUs. We report the
validation cross-entropy at the final checkpoint
for both scales in Table~\ref{tab:loop-valloss}.

\begin{table}[ht]
  \centering
  \small
  \caption{Final validation loss at step 100{,}000 on FineWeb-Edu 50BT for
    the GPT-2 small and GPT-2 medium backbones, varying loop count $R \in \{1,3,5,7\}$.
    Both methods share tied input-output embeddings and an input-embedding
    RMSNorm; the baseline uses the standard pre-LN block.
    Bold marks the best cell per column per scale; ties within $0.002$\,nats are bolded on both rows. DeepLoop is effectively tied with the baseline at $R{=}1$ and improves over it in these single-seed runs at $R{=}3$, $R{=}5$, and $R{=}7$ at both scales.}
  \label{tab:loop-valloss}
  \begin{tabular}{lcccc}
    \toprule
    Method & $R=1$ & $R=3$ & $R=5$ & $R=7$ \\
    \midrule
    \multicolumn{5}{l}{\emph{GPT-2 small backbone}} \\
    baseline (pre-LN) & \textbf{2.8627} & 2.8077 & 2.7910 & 2.7700 \\
    DeepLoop                            & \textbf{2.8631} & \textbf{2.7917} & \textbf{2.7679} & \textbf{2.7514} \\
    $\Delta$ (DeepLoop $-$ base)        & $+0.0004$ & $-0.0160$ & $-0.0231$ & $-0.0186$ \\
    \midrule
    \multicolumn{5}{l}{\emph{GPT-2 medium backbone}} \\
    baseline (pre-LN) & \textbf{2.6253} & 2.5779 & 2.5640 & 2.5558 \\
    DeepLoop                            & \textbf{2.6264} & \textbf{2.5627} & \textbf{2.5444} & \textbf{2.5280} \\
    $\Delta$ (DeepLoop $-$ base)        & $+0.0011$ & $-0.0153$ & $-0.0196$ & $-0.0278$ \\
    \bottomrule
  \end{tabular}
\end{table}

Figure~\ref{fig:loop-valloss} visualizes the same numbers as
Table~\ref{tab:loop-valloss}. At both scales the two curves coincide
within noise at $R{=}1$ ($+0.0004$ and $+0.0011$\,nats) and separate as
the loop count grows: at the small scale DeepLoop pulls ahead by
$-0.016$\,nats at $R{=}3$, the gap peaks at $-0.023$\,nats at $R{=}5$,
and remains open at $-0.019$\,nats at $R{=}7$; at the medium scale the
gap widens monotonically through $R{=}7$, reaching $-0.028$\,nats. Both
methods continue to improve monotonically with $R$ at both scales, and
DeepLoop strictly beats the baseline at every $R\geq3$. Multi-seed
runs would be needed to quantify run-to-run variance.

\begin{figure}[ht]
  \centering
  \includegraphics[width=\linewidth]{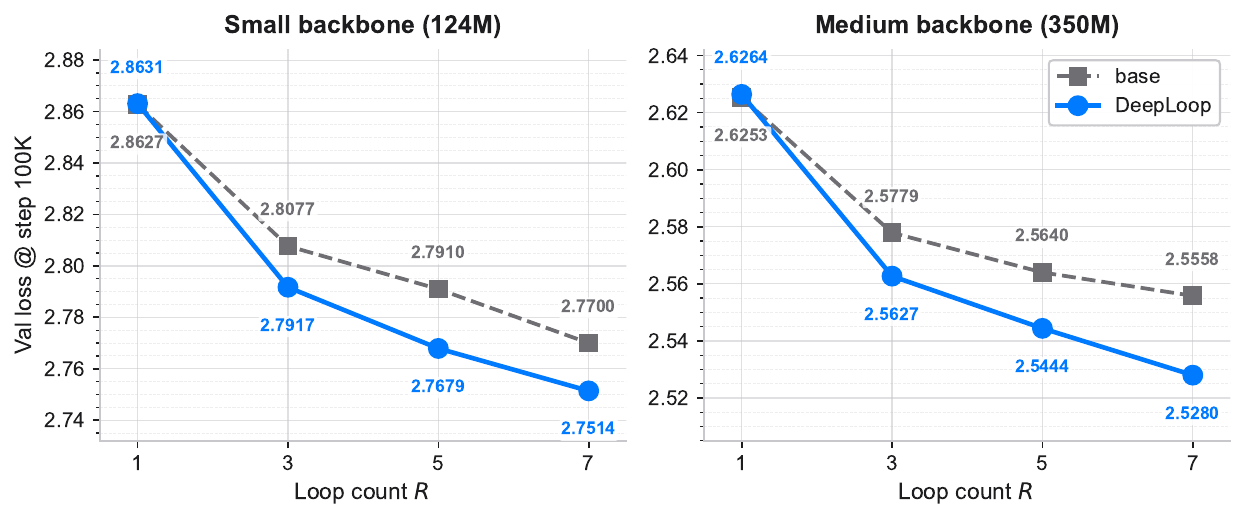}
  \caption{Final validation loss (step 100{,}000) against loop count
    $R\in\{1,3,5,7\}$ on FineWeb-Edu 50BT for the GPT-2 small (left) and
    GPT-2 medium (right) backbones. DeepLoop matches the baseline at
    $R{=}1$ and improves over it at every $R\geq3$ at both scales in
    these single-seed runs.
    Data are the same as Table~\ref{tab:loop-valloss}.}
  \label{fig:loop-valloss}
\end{figure}

\subsection{Downstream evaluation on the \texttt{lm-evaluation-harness} suite}
\label{sec:experiments:downstream}

To test whether the validation-loss advantage transfers to downstream
task accuracy, we evaluate the eight medium-scale checkpoints from
\S\ref{sec:experiments:valloss} on the eight-task zero- and one-shot
suite used by \citet{zhang2026deep} in the Deep Delta Learning paper:
\texttt{arc\_challenge}, \texttt{arc\_easy}, \texttt{hellaswag},
\texttt{openbookqa}, \texttt{piqa}, \texttt{sciq}, \texttt{social\_iqa},
and \texttt{winogrande}. Evaluation uses the
\texttt{lm-evaluation-harness} \citep{eval-harness} with its default
\texttt{acc} metric (not \texttt{acc\_norm}), the GPT-2 tokenizer, a
context length of 1024 (matching training), and bf16 inference. The
reported \emph{Avg} column is the unweighted arithmetic mean of the
eight per-task accuracies, following the protocol of
\citet{zhang2026deep}.

Table~\ref{tab:downstream-medium} reports the per-task accuracy and
Avg for every (method, $R$, shot) cell at the medium scale.
Figure~\ref{fig:downstream-per-task} breaks the 1-shot column out
task-by-task. The corresponding GPT-2 small downstream results,
which exhibit the same qualitative pattern, are deferred to
Appendix~\ref{app:small-downstream}.

\begin{table}[ht]
  \centering
  \footnotesize
  \caption{Downstream accuracy (\%) for the GPT-2 medium backbone on
    FineWeb-Edu 50BT, mirroring Table~\ref{tab:downstream-small} but
    extended to $R{=}7$. Metric is \texttt{acc}; Avg is the unweighted
    arithmetic mean across the eight tasks. Bold marks the best cell
    per column per shot setting.}
  \label{tab:downstream-medium}
  \begin{tabular}{lcccccccc c}
    \toprule
    Method / $R$ & ARC-C & ARC-E & HellaSwag & OBQA & PIQA & SciQ & SIQA & WG & \textbf{Avg} \\
    \midrule
    \multicolumn{10}{l}{\emph{0-shot}} \\
    baseline $R{=}1$       & 29.44 & 65.24 & 38.12 & 23.60 & 69.80 & 85.80 & 40.23 & 54.70 & 50.86 \\
    baseline $R{=}3$       & 32.59 & 66.75 & 39.76 & 27.60 & 70.13 & 87.30 & 39.87 & 55.96 & 52.50 \\
    baseline $R{=}5$       & \textbf{33.53} & 68.31 & 41.01 & 25.80 & 69.70 & 88.20 & 38.95 & 55.41 & 52.61 \\
    baseline $R{=}7$       & 32.34 & 69.40 & 40.61 & 24.40 & \textbf{70.67} & 89.20 & 39.66 & 57.30 & 52.95 \\
    DeepLoop $R{=}1$   & 29.86 & 65.07 & 38.21 & 23.60 & 68.82 & 86.10 & 39.92 & 55.25 & 50.85 \\
    DeepLoop $R{=}3$   & 32.59 & 68.43 & 40.85 & 26.80 & 70.29 & 88.30 & 40.02 & 56.12 & 52.93 \\
    DeepLoop $R{=}5$   & \textbf{33.53} & 68.60 & 41.78 & \textbf{28.20} & 70.46 & 90.10 & 40.58 & 56.12 & 53.67 \\
    DeepLoop $R{=}7$   & 33.11 & \textbf{69.65} & \textbf{41.83} & 25.60 & 70.57 & \textbf{90.30} & \textbf{40.94} & \textbf{59.04} & \textbf{53.88} \\
    \midrule
    \multicolumn{10}{l}{\emph{1-shot}} \\
    baseline $R{=}1$       & 30.38 & 66.92 & 37.73 & 24.20 & 70.08 & 88.30 & 40.79 & 55.09 & 51.69 \\
    baseline $R{=}3$       & 33.70 & 68.69 & 39.71 & 26.00 & 70.13 & 90.50 & 40.74 & 55.33 & 53.10 \\
    baseline $R{=}5$       & 35.75 & 70.88 & 40.82 & 27.80 & \textbf{71.60} & 91.90 & 39.71 & 56.91 & 54.42 \\
    baseline $R{=}7$       & 34.22 & 70.54 & 40.87 & 27.80 & 71.33 & 92.10 & 41.30 & 58.80 & 54.62 \\
    DeepLoop $R{=}1$   & 30.80 & 65.78 & 37.73 & 22.80 & 69.10 & 89.50 & 41.30 & 55.25 & 51.53 \\
    DeepLoop $R{=}3$   & 33.11 & 68.14 & 40.59 & 26.20 & 69.48 & 91.30 & 41.15 & 57.38 & 53.42 \\
    DeepLoop $R{=}5$   & \textbf{36.43} & 68.18 & 41.50 & 26.00 & 70.84 & 92.10 & 40.99 & 57.46 & 54.19 \\
    DeepLoop $R{=}7$   & 35.15 & \textbf{70.96} & \textbf{41.83} & \textbf{28.80} & 70.51 & \textbf{92.40} & \textbf{42.53} & \textbf{59.43} & \textbf{55.20} \\
    \bottomrule
  \end{tabular}
\end{table}

\begin{figure}[ht]
  \centering
  \includegraphics[width=0.95\linewidth]{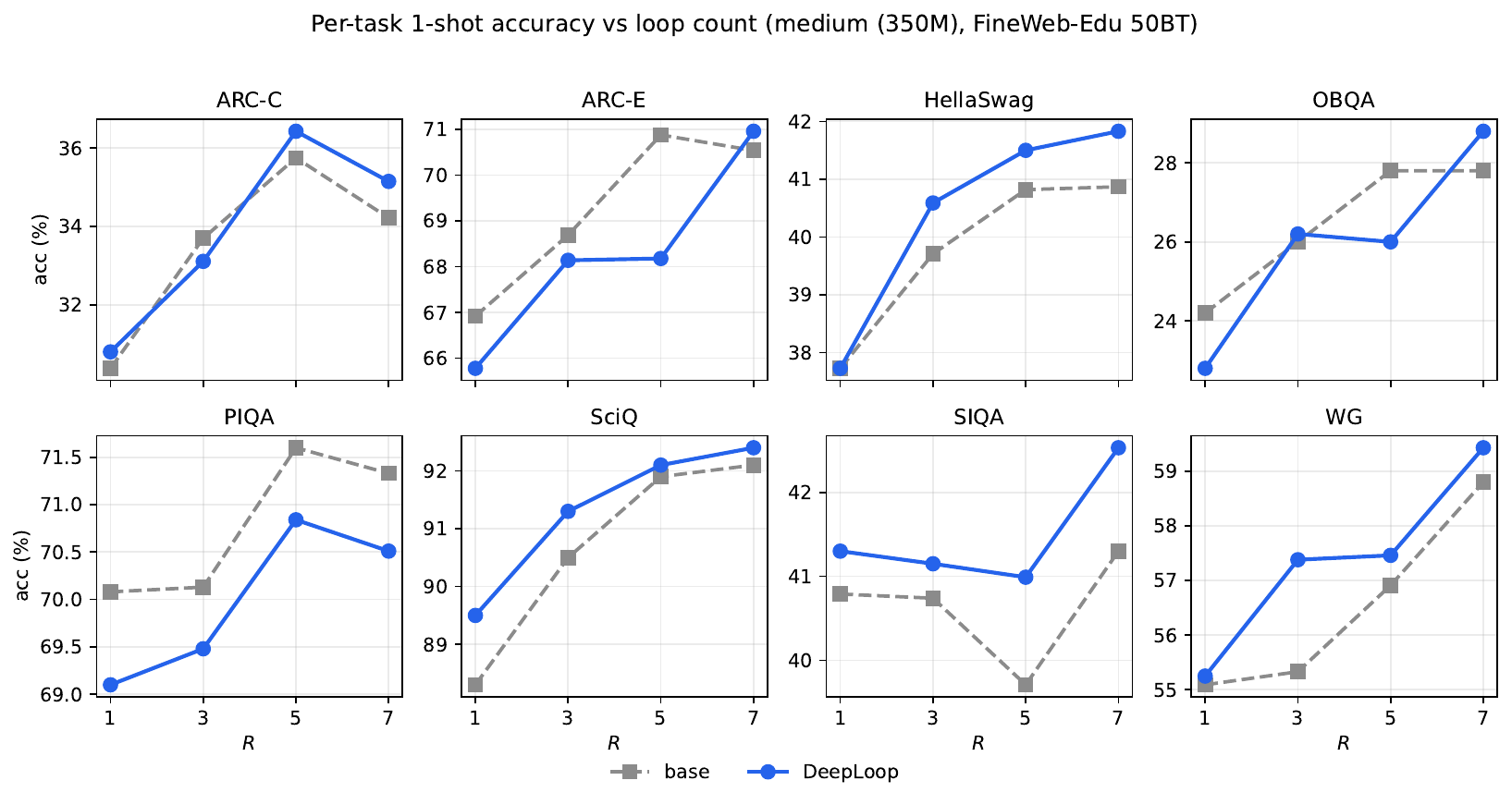}
  \caption{Per-task 1-shot accuracy against loop count $R$ for base
    versus DeepLoop on all eight tasks at the GPT-2 medium scale. At
    $R{=}7$ DeepLoop wins on seven of the eight tasks under both shot
    settings; see Table~\ref{tab:downstream-medium} for exact values.}
  \label{fig:downstream-per-task}
\end{figure}

Two takeaways. First, the qualitative picture from the small-scale
ablation (Appendix~Table~\ref{tab:downstream-small}) carries over to
the larger backbone: the two methods are tied at $R{=}1$
($\Delta=-0.01$ on the 0-shot Avg, $-0.16$ on 1-shot), DeepLoop opens
the gap at $R{=}3$ ($+0.43$ / $+0.32$), and the best Avg cell in the
table is DeepLoop $R{=}7$ in both shot settings, headlined by
$55.20\,\%$ on 1-shot. On the 0-shot Avg the gap opens at $R{=}3$ and
stays around a point at $R{=}5$ and $R{=}7$ ($+0.43$, $+1.06$, $+0.93$); on the 1-shot Avg the
ranking is messier at $R{=}5$ (where the base reaches $54.42\,\%$,
$0.23$ points above DeepLoop) but DeepLoop recovers the lead by
$+0.58$ points at $R{=}7$. Second, at $R{=}7$ DeepLoop beats the
same-$R$ baseline on seven of the eight individual tasks under both
shot settings (only PIQA goes the other way), with WinoGrande in
particular jumping by $+1.74$ points 0-shot ($59.04$ vs.\ $57.30$).
The same caveat about single-seed run-to-run variance from
\S\ref{sec:experiments:valloss} applies, we see no qualitative
disagreement with the small-scale ablation, only a stronger depth
signal.

\subsection{Reasoning evaluation on ARC-AGI}
\label{sec:experiments:arc}

To test the prediction of Eq.~(\ref{eq:hrm-p-prediction}) on a hierarchical
recurrent reasoner, we apply the DeepLoop scaling rule with $p=1/2$ to the
Hierarchical Reasoning Model (HRM) \citep{wang2025hrm} and evaluate on
ARC-AGI-1 \citep{chollet2019measure}. The only change from the published HRM is
the residual parameterization: every reasoning block uses $\alpha=(2N_{\mathrm{g}})^{1/2}$ and $\beta=(8N_{\mathrm{g}})^{-1/2}$ with $M_{\mathrm{g}}=2\,(K_H+C_LK_L)=24$ gradient-visible residual-sublayer visits for the published configuration ($K_H{=}K_L{=}4$ blocks per module, $C_L{=}2$ inner iterations) and $N_{\mathrm{g}}=M_{\mathrm{g}}/2=12$ block-equivalent depth (Eq.~\ref{eq:hrm-Mg}); the
backbone, the adaptive-computation halting, the optimizer (AdamATan2), the data
pipeline, and the 100K-epoch schedule are held fixed. Baseline and DeepLoop
runs train on an identical, hash-verified \texttt{arc-aug-1000} build and are
scored by the same count-first voting protocol on $N=400$ evaluation puzzles.
As a harness check, the published HRM ARC-2 checkpoint scores $5.00\%$
($K{=}2$) on this pipeline, matching the released value.

Table~\ref{tab:hrm-arc} reports voted accuracy at voting budgets
$K\in\{1,2,10,100,1000\}$. DeepLoop improves the paper-protocol two-vote
accuracy from $36.50\%$ to $39.75\%$ ($+3.25$\,pp) and improves every column of
the voting ladder. A four-seed control places the per-seed $K{=}2$ standard
deviation at $\approx 0.5$\,pp, so the $+3.25$\,pp gain is roughly a $6\sigma$
effect rather than a seed draw. The prediction that a tied hierarchical reasoner
sits at the aligned $p=1/2$ threshold, the same exponent that the worst-case
analysis prescribes for the single-module looped backbone
(Section~\ref{sec:method:loop}), is therefore
borne out: a single residual-scaling rule, $p=1/2$, is the correct setting for
both the looped language model and the hierarchical recurrent reasoner.

\begin{table}[ht]
  \centering
  \small
  \caption{Voted accuracy (\%) on ARC-AGI-1 for the vanilla Hierarchical
    Reasoning Model and the same model with the DeepLoop residual scaling
    ($p=1/2$), at voting budgets $K$. Both runs use 100K epochs, AdamATan2, and
    an identical hash-verified \texttt{arc-aug-1000} build; $N=400$ evaluation
    puzzles. Bold marks the better cell per column; the $K{=}2$ column is the
    paper-protocol headline metric.}
  \label{tab:hrm-arc}
  \begin{tabular}{lccccc}
    \toprule
    Method & $K{=}1$ & $K{=}2$ & $K{=}10$ & $K{=}100$ & $K{=}1000$ \\
    \midrule
    Vanilla HRM           & 31.50 & 36.50 & 41.50 & 47.50 & 50.75 \\
    DeepLoop ($p{=}1/2$)  & \textbf{35.50} & \textbf{39.75} & \textbf{44.25} & \textbf{49.75} & \textbf{51.50} \\
    $\Delta$              & $+4.00$ & $+3.25$ & $+2.75$ & $+2.25$ & $+0.75$ \\
    \bottomrule
  \end{tabular}
\end{table}

\section{Related Work}
\label{sec:related}

\paragraph{Looped and universal Transformers.}
Universal Transformers \citep{dehghani2018universal} and ALBERT
\citep{lan2019albert} introduced depth-wise parameter sharing as a way to
increase effective computation without increasing parameter count. Later work
studied partial sharing \citep{reid2021subformer}, looped Transformers for
in-context learning and algorithmic computation \citep{giannou2023looped,
yang2023looped,gatmiry2024can}, and recurrent-depth language models that trade
inference compute for quality \citep{geiping2025scaling}. DeepLoop is
complementary to these architectures: within the post-normalized sandwich
block of Eq.~(\ref{eq:deeploop-block}), it prescribes the residual and
initialization scales used when a physical block is revisited.

\paragraph{Adaptive computation and iterative reasoning.}
Looped depth can also be viewed as a differentiable test-time-compute
mechanism. Adaptive Computation Time \citep{graves2016adaptive}, PonderNet
\citep{banino2021pondernet}, and recurrent networks for algorithmic
extrapolation \citep{schwarzschild2021can,bansal2022end} all exploit repeated
application of shared computation. Recent language-modeling work similarly adds
latent or recurrent computation \citep{goyal2023think,hao2024training,
saunshi2025reasoning}. Our analysis targets the stability of this repeated
shared computation.

\paragraph{Residual scaling and parameterization.}
Transformer stability has been improved through normalization placement
\citep{xiong2020layer,nguyen2019transformers}, initialization and residual
scaling \citep{zhang2019fixup,huang2020improving,de2020batch,
bachlechner2021rezero,liu2020understanding}, and DeepNorm/DeepNet
\citep{wang2024deepnet}, which derives $\alpha=(2N)^{1/4}$ and
$\beta=(8N)^{-1/4}$ for untied Post-LN Transformers. Width-and-depth
parameterizations such as $\mu$P and Depth-$\mu$P prescribe complementary
scalings for feature learning and hyperparameter transfer
\citep{yang2021tensor,yang2021tuning,yang2023tensor,bordelon2023depthwise}.
These analyses do not explicitly account for the extra correlation introduced when the same residual-branch parameters are reused across loop visits.
DeepLoop is a loop-specific correction to the DeepNorm exponent.

\section{Conclusion}

Looped Transformers turn depth into a controllable compute resource, but repeated parameter reuse changes the residual-scaling problem. When a residual branch is revisited, its shared update is accumulated across visits and then read by those same visits in the next linearized forward pass. We captured this tied-depth effect with a visit-alignment coefficient $\kappa_R$, recovering the DeepNorm exponent for decorrelated visits and identifying $p=1/2$ as the conservative exponent for aligned visits at fixed physical depth.

DeepLoop implements this correction in the post-normalized sandwich block of
Eq.~(\ref{eq:deeploop-block}) with a one-line scaling rule,
$\alpha=(2N)^{1/2}$ and $\beta=(8N)^{-1/2}$. Empirically, it is neutral when no physical block is revisited and improves looped GPT-style language models as recurrent depth increases, with gains in validation loss and downstream accuracy. A $p$-sweep at $R{=}3$ further places the empirical stability transition in the vicinity of $p=1/2$.

Future work should measure $\kappa_R$ or cross-round gradient alignment directly, test whether the same boundary holds at larger scale or under alternative parameterizations, and explore whether training can encourage decorrelated visits to safely use less conservative exponents.

\section*{Acknowledgement}

We thank Zixuan Wang and Eric Song for their constructful feedback and insightful discussion.

\vspace{5ex}
\bibliographystyle{plainnat}
\bibliography{reference}

@misc{zhu2026scalinglatentreasoninglooped,
      title={Scaling Latent Reasoning via Looped Language Models}, 
      author={Rui-Jie Zhu and Zixuan Wang and Kai Hua and Tianyu Zhang and Ziniu Li and Haoran Que and Boyi Wei and Zixin Wen and Fan Yin and He Xing and Lu Li and Jiajun Shi and Kaijing Ma and Shanda Li and Taylor Kergan and Andrew Smith and Xingwei Qu and Mude Hui and Bohong Wu and Qiyang Min and Hongzhi Huang and Xun Zhou and Wei Ye and Jiaheng Liu and Jian Yang and Yunfeng Shi and Chenghua Lin and Enduo Zhao and Tianle Cai and Ge Zhang and Wenhao Huang and Yoshua Bengio and Jason Eshraghian},
      year={2026},
      eprint={2510.25741},
      archivePrefix={arXiv},
      primaryClass={cs.CL},
      url={https://arxiv.org/abs/2510.25741}, 
}

@misc{fu2026discolooploopingdiscreteembeddings,
      title={DiscoLoop: Looping Discrete Embeddings and Continuous Hidden States for Multi-hop Reasoning}, 
      author={Hengyu Fu and Tianyu Guo and Zixuan Wang and Hanlin Zhu and Jason D. Lee and Jiantao Jiao and Stuart Russell and Song Mei},
      year={2026},
      eprint={2607.00341},
      archivePrefix={arXiv},
      primaryClass={cs.CL},
      url={https://arxiv.org/abs/2607.00341}, 
}

@article{wang2024deepnet,
  title={Deepnet: Scaling transformers to 1,000 layers},
  author={Wang, Hongyu and Ma, Shuming and Dong, Li and Huang, Shaohan and Zhang, Dongdong and Wei, Furu},
  journal={IEEE Transactions on Pattern Analysis and Machine Intelligence},
  volume={46},
  number={10},
  pages={6761--6774},
  year={2024},
  publisher={IEEE}
}

@article{zhang2019fixup,
  title={Fixup initialization: Residual learning without normalization},
  author={Zhang, Hongyi and Dauphin, Yann N and Ma, Tengyu},
  journal={arXiv preprint arXiv:1901.09321},
  year={2019}
}

@article{dehghani2018universal,
  title={Universal transformers},
  author={Dehghani, Mostafa and Gouws, Stephan and Vinyals, Oriol and Uszkoreit, Jakob and Kaiser, {\L}ukasz},
  journal={arXiv preprint arXiv:1807.03819},
  year={2018}
}

@article{lan2019albert,
  title={Albert: A lite bert for self-supervised learning of language representations},
  author={Lan, Zhenzhong and Chen, Mingda and Goodman, Sebastian and Gimpel, Kevin and Sharma, Piyush and Soricut, Radu},
  journal={arXiv preprint arXiv:1909.11942},
  year={2019}
}

@inproceedings{reid2021subformer,
  title={Subformer: Exploring weight sharing for parameter efficiency in generative transformers},
  author={Reid, Machel and Marrese-Taylor, Edison and Matsuo, Yutaka},
  booktitle={Findings of the Association for Computational Linguistics: EMNLP 2021},
  pages={4081--4090},
  year={2021}
}

@inproceedings{giannou2023looped,
  title={Looped transformers as programmable computers},
  author={Giannou, Angeliki and Rajput, Shashank and Sohn, Jy-yong and Lee, Kangwook and Lee, Jason D and Papailiopoulos, Dimitris},
  booktitle={International Conference on Machine Learning},
  pages={11398--11442},
  year={2023},
  organization={PMLR}
}

@article{yang2023looped,
  title={Looped transformers are better at learning learning algorithms},
  author={Yang, Liu and Lee, Kangwook and Nowak, Robert and Papailiopoulos, Dimitris},
  journal={arXiv preprint arXiv:2311.12424},
  year={2023}
}

@article{gatmiry2024can,
  title={Can looped transformers learn to implement multi-step gradient descent for in-context learning?},
  author={Gatmiry, Khashayar and Saunshi, Nikunj and Reddi, Sashank J and Jegelka, Stefanie and Kumar, Sanjiv},
  journal={arXiv preprint arXiv:2410.08292},
  year={2024}
}

@article{saunshi2025reasoning,
  title={Reasoning with latent thoughts: On the power of looped transformers},
  author={Saunshi, Nikunj and Dikkala, Nishanth and Li, Zhiyuan and Kumar, Sanjiv and Reddi, Sashank J},
  journal={arXiv preprint arXiv:2502.17416},
  year={2025}
}

@article{graves2016adaptive,
  title={Adaptive computation time for recurrent neural networks},
  author={Graves, Alex},
  journal={arXiv preprint arXiv:1603.08983},
  year={2016}
}

@article{banino2021pondernet,
  title={Pondernet: Learning to ponder},
  author={Banino, Andrea and Balaguer, Jan and Blundell, Charles},
  journal={arXiv preprint arXiv:2107.05407},
  year={2021}
}

@article{schwarzschild2021can,
  title={Can you learn an algorithm? generalizing from easy to hard problems with recurrent networks},
  author={Schwarzschild, Avi and Borgnia, Eitan and Gupta, Arjun and Huang, Furong and Vishkin, Uzi and Goldblum, Micah and Goldstein, Tom},
  journal={Advances in Neural Information Processing Systems},
  volume={34},
  pages={6695--6706},
  year={2021}
}

@article{bansal2022end,
  title={End-to-end algorithm synthesis with recurrent networks: Extrapolation without overthinking},
  author={Bansal, Arpit and Schwarzschild, Avi and Borgnia, Eitan and Emam, Zeyad and Huang, Furong and Goldblum, Micah and Goldstein, Tom},
  journal={Advances in Neural Information Processing Systems},
  volume={35},
  pages={20232--20242},
  year={2022}
}

@article{geiping2025scaling,
  title={Scaling up test-time compute with latent reasoning: A recurrent depth approach},
  author={Geiping, Jonas and McLeish, Sean and Jain, Neel and Kirchenbauer, John and Singh, Siddharth and Bartoldson, Brian R and Kailkhura, Bhavya and Bhatele, Abhinav and Goldstein, Tom},
  journal={arXiv preprint arXiv:2502.05171},
  year={2025}
}

@article{wang2025hrm,
  title={Hierarchical Reasoning Model},
  author={Wang, Guan and Li, Jin and Sun, Yuhao and Chen, Xing and Liu, Changling and Wu, Yue and Lu, Meng and Song, Sen and Yadkori, Yasin Abbasi},
  journal={arXiv preprint arXiv:2506.21734},
  year={2025}
}

@article{chollet2019measure,
  title={On the Measure of Intelligence},
  author={Chollet, Fran\c{c}ois},
  journal={arXiv preprint arXiv:1911.01547},
  year={2019}
}

@article{hao2024training,
  title={Training large language models to reason in a continuous latent space},
  author={Hao, Shibo and Sukhbaatar, Sainbayar and Su, DiJia and Li, Xian and Hu, Zhiting and Weston, Jason and Tian, Yuandong},
  journal={arXiv preprint arXiv:2412.06769},
  year={2024}
}

@article{goyal2023think,
  title={Think before you speak: Training language models with pause tokens},
  author={Goyal, Sachin and Ji, Ziwei and Rawat, Ankit Singh and Menon, Aditya Krishna and Kumar, Sanjiv and Nagarajan, Vaishnavh},
  journal={arXiv preprint arXiv:2310.02226},
  year={2023}
}

@inproceedings{liu2020understanding,
  title={Understanding the difficulty of training transformers},
  author={Liu, Liyuan and Liu, Xiaodong and Gao, Jianfeng and Chen, Weizhu and Han, Jiawei},
  booktitle={Proceedings of the 2020 Conference on Empirical Methods in Natural Language Processing (EMNLP)},
  pages={5747--5763},
  year={2020}
}

@inproceedings{bachlechner2021rezero,
  title={Rezero is all you need: Fast convergence at large depth},
  author={Bachlechner, Thomas and Majumder, Bodhisattwa Prasad and Mao, Henry and Cottrell, Gary and McAuley, Julian},
  booktitle={Uncertainty in artificial intelligence},
  pages={1352--1361},
  year={2021},
  organization={PMLR}
}

@inproceedings{huang2020improving,
  title={Improving transformer optimization through better initialization},
  author={Huang, Xiao Shi and Perez, Felipe and Ba, Jimmy and Volkovs, Maksims},
  booktitle={International Conference on Machine Learning},
  pages={4475--4483},
  year={2020},
  organization={PMLR}
}

@article{de2020batch,
  title={Batch normalization biases residual blocks towards the identity function in deep networks},
  author={De, Soham and Smith, Sam},
  journal={Advances in Neural Information Processing Systems},
  volume={33},
  pages={19964--19975},
  year={2020}
}

@inproceedings{xiong2020layer,
  title={On layer normalization in the transformer architecture},
  author={Xiong, Ruibin and Yang, Yunchang and He, Di and Zheng, Kai and Zheng, Shuxin and Xing, Chen and Zhang, Huishuai and Lan, Yanyan and Wang, Liwei and Liu, Tieyan},
  booktitle={International conference on machine learning},
  pages={10524--10533},
  year={2020},
  organization={PMLR}
}

@inproceedings{nguyen2019transformers,
  title={Transformers without tears: Improving the normalization of self-attention},
  author={Nguyen, Toan Q and Salazar, Julian},
  booktitle={Proceedings of the 16th international conference on spoken language translation},
  year={2019}
}

@inproceedings{yang2021tensor,
  title={Tensor programs iv: Feature learning in infinite-width neural networks},
  author={Yang, Greg and Hu, Edward J},
  booktitle={International Conference on Machine Learning},
  pages={11727--11737},
  year={2021},
  organization={PMLR}
}

@article{yang2021tuning,
  title={Tuning large neural networks via zero-shot hyperparameter transfer},
  author={Yang, Ge and Hu, Edward and Babuschkin, Igor and Sidor, Szymon and Liu, Xiaodong and Farhi, David and Ryder, Nick and Pachocki, Jakub and Chen, Weizhu and Gao, Jianfeng},
  journal={Advances in Neural Information Processing Systems},
  volume={34},
  pages={17084--17097},
  year={2021}
}

@article{yang2023tensor,
  title={Tensor programs vi: Feature learning in infinite-depth neural networks},
  author={Yang, Greg and Yu, Dingli and Zhu, Chen and Hayou, Soufiane},
  journal={arXiv preprint arXiv:2310.02244},
  year={2023}
}

@article{bordelon2023depthwise,
  title={Depthwise hyperparameter transfer in residual networks: Dynamics and scaling limit},
  author={Bordelon, Blake and Noci, Lorenzo and Li, Mufan Bill and Hanin, Boris and Pehlevan, Cengiz},
  journal={arXiv preprint arXiv:2309.16620},
  year={2023}
}

@article{zhang2026deep,
  title={Deep delta learning},
  author={Zhang, Yifan and Liu, Yifeng and Wang, Mengdi and Gu, Quanquan},
  journal={arXiv preprint arXiv:2601.00417},
  year={2026}
}

@misc{eval-harness,
  author       = {Gao, Leo and Tow, Jonathan and Abbasi, Baber and Biderman, Stella and Black, Sid and DiPofi, Anthony and Foster, Charles and Golding, Laurence and Hsu, Jeffrey and Le Noac'h, Alain and Li, Haonan and McDonell, Kyle and Muennighoff, Niklas and Ociepa, Chris and Phang, Jason and Reynolds, Laria and Schoelkopf, Hailey and Skowron, Aviya and Sutawika, Lintang and Tang, Eric and Thite, Anish and Wang, Ben and Wang, Kevin and Zou, Andy},
  title        = {The Language Model Evaluation Harness},
  month        = 07,
  year         = 2024,
  publisher    = {Zenodo},
  version      = {v0.4.3},
  doi          = {10.5281/zenodo.12608602},
  url          = {https://zenodo.org/records/12608602}
}

@inproceedings{vaswani2017attention,
  title={Attention is all you need},
  author={Vaswani, Ashish and Shazeer, Noam and Parmar, Niki and Uszkoreit, Jakob and Jones, Llion and Gomez, Aidan N and Kaiser, {\L}ukasz and Polosukhin, Illia},
  booktitle={Advances in Neural Information Processing Systems},
  volume={30},
  year={2017}
}

@article{kaplan2020scaling,
  title={Scaling laws for neural language models},
  author={Kaplan, Jared and McCandlish, Sam and Henighan, Tom and Brown, Tom B and Chess, Benjamin and Child, Rewon and Gray, Scott and Radford, Alec and Wu, Jeffrey and Amodei, Dario},
  journal={arXiv preprint arXiv:2001.08361},
  year={2020}
}

@article{hoffmann2022training,
  title={Training compute-optimal large language models},
  author={Hoffmann, Jordan and Borgeaud, Sebastian and Mensch, Arthur and Buchatskaya, Elena and Cai, Trevor and Rutherford, Eliza and Casas, Diego de Las and Hendricks, Lisa Anne and Welbl, Johannes and Clark, Aidan and Hennigan, Tom and Noland, Eric and Millican, Katherine and van den Driessche, George and Damoc, Bogdan and Guy, Aurelia and Osindero, Simon and Simonyan, Karen and Elsen, Erich and Rae, Jack W and Vinyals, Oriol and Sifre, Laurent},
  journal={arXiv preprint arXiv:2203.15556},
  year={2022}
}

\clearpage
\appendix

\renewcommand{\appendixpagename}{\centering \huge Appendix}
\appendixpage
\counterwithin{theorem}{section}

\startcontents[section]
\printcontents[section]{l}{1}{\setcounter{tocdepth}{2}}
\clearpage

\section{Proofs}
\label{app:proofs}

\begin{lemma}[RMSNorm exposes the residual branch through $1/\alpha$]
\label{lem:rmsnorm-local}
Let $\mathrm{RMS}(x)^2=d^{-1}\|x\|^2$, assume $\mathrm{RMS}(x)=1$, and define
$\mathcal{R}(y)=y/\mathrm{RMS}(y)$. If $\mathrm{RMS}(z)/\alpha\le c<1$, then
\begin{equation}
  \mathcal{R}(\alpha x+z)
  =
  x
  +
  \frac{z-\langle x,z\rangle_d x}{\alpha}
  +
  O\left(\frac{\mathrm{RMS}(z)^2}{\alpha^2}\right),
  \qquad
  \langle x,z\rangle_d := d^{-1}x^\top z .
  \label{eq:rmsnorm-expansion}
\end{equation}
\end{lemma}

\begin{proof}[Proof of Lemma~\ref{lem:rmsnorm-local}]
Write $m=\langle x,z\rangle_d$ and
$\delta^2=\mathrm{RMS}(z)^2$. Since $|m|\le\delta$ and
$\delta/\alpha\le c<1$, the quantity below is bounded away from zero:
\[
  \mathrm{RMS}(\alpha x+z)
  =\alpha\left(1+2m/\alpha+\delta^2/\alpha^2\right)^{1/2}.
\]
A first-order Taylor expansion of the inverse square root around one gives
\[
  \left(1+2m/\alpha+\delta^2/\alpha^2\right)^{-1/2}
  =1-m/\alpha+O(\delta^2/\alpha^2),
\]
where the remainder is uniform for $\delta/\alpha\le c$. Multiplying by
$(\alpha x+z)/\alpha$ yields
\[
  \frac{\alpha x+z}{\mathrm{RMS}(\alpha x+z)}
  =x+\frac{z-mx}{\alpha}+O(\delta^2/\alpha^2),
\]
which is Eq.~(\ref{eq:rmsnorm-expansion}).
\end{proof}

\begin{lemma}[Inner normalization is inert on a post-normalized stream]
\label{lem:inner-norm}
Let $\mathcal{R}(y)=y/\mathrm{RMS}(y)$ as in Lemma~\ref{lem:rmsnorm-local}.
\emph{(i)} If $\mathrm{RMS}(x)=1$, then $\mathcal{R}(x)=x$.
\emph{(ii)} For any $x$ with $\mathrm{RMS}(x)>0$ and $d\ge 2$, the Jacobian
of $\mathcal{R}$ is
\[
  \partial\mathcal{R}(x)
  =
  \frac{1}{\mathrm{RMS}(x)}
  \left(I-\hat{x}\hat{x}^{\top}\right),
  \qquad
  \hat{x}:=x/\|x\|,
\]
an orthogonal projection scaled by $1/\mathrm{RMS}(x)$, so its operator norm
is $1/\mathrm{RMS}(x)$; in particular it equals one when
$\mathrm{RMS}(x)=1$.
\end{lemma}

\begin{proof}[Proof of Lemma~\ref{lem:inner-norm}]
Part \emph{(i)} is immediate from the definition. For \emph{(ii)}, write
$\mathcal{R}(y)=\sqrt{d}\,y/\|y\|$ and differentiate:
\[
  \partial\mathcal{R}(x)
  =
  \frac{\sqrt{d}}{\|x\|}
  \left(I-\frac{x x^{\top}}{\|x\|^{2}}\right)
  =
  \frac{1}{\mathrm{RMS}(x)}
  \left(I-\hat{x}\hat{x}^{\top}\right),
\]
using $\|x\|=\sqrt{d}\,\mathrm{RMS}(x)$. The matrix
$I-\hat{x}\hat{x}^{\top}$ is the orthogonal projection onto the complement of
$\mathrm{span}\{x\}$, with eigenvalues in $\{0,1\}$, so the operator norm of
$\partial\mathcal{R}(x)$ is $1/\mathrm{RMS}(x)$.
\end{proof}

Lemma~\ref{lem:inner-norm} transfers the analysis to the sandwich block of
Eq.~(\ref{eq:deeploop-block}). At initialization the normalization gains are
one, the loop entry has unit RMS by the input-embedding normalization, and
every later state is an output of the outer normalization, so
$\mathrm{RMS}(\mathbf{x}_i)=1$ at every visit. By part \emph{(i)} the inner
normalization then evaluates the branch, and hence every parameter derivative
$\partial f_j/\partial\phi_j$, at the unchanged input
$\mathcal{R}(\mathbf{x}_i)=\mathbf{x}_i$; by part \emph{(ii)} each linearized
state-propagation path acquires an additional factor of operator norm one at
each branch traversal. The visit-wise constants of
Assumption~\ref{ass:local-sensitivity}, and the constant $C'$ of
Proposition~\ref{prop:untied}, are therefore unchanged or smaller for the
sandwich block. The tied bound of Proposition~\ref{prop:tied-perturbation} is
stated through the alignment coefficient $\kappa_R$ of
Eq.~(\ref{eq:kappa-def}), which is defined for the architecture at hand:
because the sandwich block satisfies the same visit-wise bounds
$\|U_{r,j}\|\le C_U\beta/\alpha$ and $\|G_{r,j}\|\le C_G\beta/\alpha$, the
triangle inequality again gives $0\le\kappa_R\le R$, and
Proposition~\ref{prop:tied-perturbation},
Corollary~\ref{cor:aligned-loop}, and
Proposition~\ref{prop:exponent-threshold} hold for
Eq.~(\ref{eq:deeploop-block}) exactly as stated, with the worst-case aligned
regime $\kappa_R=\Theta(R)$, and hence the $p=1/2$ exponent, unchanged. The
hierarchical reasoners of Section~\ref{sec:hrm} use the plain post-normalized
block without inner normalization, for which the propositions apply
directly.

\begin{proposition}[Depth-untied first-order stability]
\label{prop:untied}
For a depth-$N$ post-normalized Transformer whose residual-sublayer parameters
are not shared across unrolled depth, with $M=2N$ residual-sublayer visits,
Assumption~\ref{ass:local-sensitivity} implies the first-order bound
Eq.~(\ref{eq:update-bound}), and consequently $M(\beta/\alpha)^2=O(1)$
(Eq.~(\ref{eq:deepnet-condition})) is a sufficient first-order stability
condition. This statement does not depend on whether the input and output token
embeddings are tied.
\end{proposition}

\begin{proof}[Proof of Proposition~\ref{prop:untied}]
In the depth-untied residual stack, every unrolled residual-sublayer visit has
its own residual-branch parameter tensor. The first-order change is therefore a
sum of $M$ independent visit-wise
perturbation terms. By Assumption~\ref{ass:local-sensitivity}, each term is
bounded by an output sensitivity $O(\beta/\alpha)$ times an effective update
$O(\beta/\alpha)$. Summing over $M$ visits gives
\[
  \|\Delta F\|
  \le C' M (\beta/\alpha)^2,
\]
for a depth-independent constant $C'$, which implies the stated sufficient
condition.
\end{proof}

\begin{proposition}[Tied loop perturbation]
\label{prop:tied-perturbation}
Under Assumption~\ref{ass:local-sensitivity}, a looped Transformer with
$J=2K$ physical residual sublayers and $R$ visits satisfies the first-order
bound Eq.~(\ref{eq:loop-update-bound}), so
$M\kappa_R(\beta/\alpha)^2=O(1)$ (Eq.~(\ref{eq:loop-condition})) is a sufficient tied-depth stability condition.
\end{proposition}

\begin{corollary}[Worst-case aligned loop bound]
\label{cor:aligned-loop}
If $\kappa_R=\Theta(R)$, then Eq.~(\ref{eq:loop-condition}) reduces to
$MR(\beta/\alpha)^2 = O(1)$.
\end{corollary}

\begin{proof}[Proof of Proposition~\ref{prop:tied-perturbation}]
Using Eq.~(\ref{eq:shared-grad}) and submultiplicativity,
\[
  \|\Delta F_{\mathrm{tied}}\|
  \le
  \eta\sum_{j=1}^{J}
  \left\|\sum_{r=1}^{R} U_{r,j}\right\|
  \left\|\sum_{t=1}^{R} G_{t,j}\right\|
  +O(\eta^2).
\]
By the definition of $\kappa_R$, each physical sublayer contributes at most
$R\kappa_R C_UC_G(\beta/\alpha)^2$. Summing over $J$ physical sublayers gives
\[
  JR\kappa_R C_UC_G(\beta/\alpha)^2
  =M\kappa_R C_UC_G(\beta/\alpha)^2.
\]
Absorbing $\eta$, $C_U$, and $C_G$ into the constant proves
Eq.~(\ref{eq:loop-update-bound}).
\end{proof}

\begin{proof}[Proof of Proposition~\ref{prop:exponent-threshold}]
Substituting Eq.~(\ref{eq:beta-over-alpha}) into the tied-depth condition gives
\[
  M\kappa_R\left(\frac{\beta}{\alpha}\right)^2
  =2N\,\Theta(R^\gamma)\,(cd)^{-2p}N^{-4p}.
\]
At fixed physical depth $K$, $R=N/K$, so the $N$-dependence is
$\Theta(N^{1+\gamma-4p})$. This quantity remains uniformly bounded as
$R\to\infty$ if and only if $1+\gamma-4p\le0$, equivalently
$p\ge(1+\gamma)/4$.
\end{proof}


\clearpage

\section{Small-scale downstream evaluation}
\label{app:small-downstream}

We report the GPT-2 small version of the downstream
evaluation from \S\ref{sec:experiments:downstream}. The architecture
is identical to the medium backbone except for hidden size $768$
(vs.\ $1024$) and $12$ layers (vs.\ $24$); training uses the same
FineWeb-Edu 50BT data and schedule but on $4\times$H200~141\,GB GPUs.
The same eight-task \texttt{lm-evaluation-harness} protocol applies.
Table~\ref{tab:downstream-small} reports per-task accuracy and Avg;
Figures~\ref{fig:downstream-avg-small}
and~\ref{fig:downstream-per-task-small} visualize the 1-shot Avg and the per-task 1-shot accuracy, respectively.

\begin{table}[ht]
  \centering
  \footnotesize
  \caption{Downstream accuracy (\%) for the GPT-2 small backbone.
    Metric is \texttt{acc}; Avg is the unweighted arithmetic mean
    across the eight tasks. Bold marks the best cell per column per
    shot setting.}
  \label{tab:downstream-small}
  \begin{tabular}{lcccccccc c}
    \toprule
    Method / $R$ & ARC-C & ARC-E & HellaSwag & OBQA & PIQA & SciQ & SIQA & WG & \textbf{Avg} \\
    \midrule
    \multicolumn{10}{l}{\emph{0-shot}} \\
    base $R{=}1$         & 23.12 & 57.11 & 31.95 & 20.60 & 65.83 & 81.90 & 37.87 & 51.62 & 46.25 \\
    base $R{=}3$         & 23.89 & 60.40 & 33.70 & 20.20 & 66.76 & 83.20 & 37.21 & 52.09 & 47.18 \\
    base $R{=}5$         & 26.62 & \textbf{61.49} & 34.14 & 20.80 & 66.49 & \textbf{83.80} & 38.54 & 53.12 & 48.12 \\
    DeepLoop $R{=}1$     & 24.49 & 58.38 & 32.31 & 20.40 & 65.83 & 80.90 & 38.23 & 51.30 & 46.48 \\
    DeepLoop $R{=}3$     & \textbf{26.96} & 60.31 & 33.98 & \textbf{22.00} & 67.79 & \textbf{83.80} & \textbf{39.87} & 53.20 & \textbf{48.49} \\
    DeepLoop $R{=}5$     & 25.09 & 61.36 & \textbf{34.86} & 21.00 & \textbf{68.50} & 82.20 & 39.46 & \textbf{53.75} & 48.28 \\
    \midrule
    \multicolumn{10}{l}{\emph{1-shot}} \\
    base $R{=}1$         & 26.11 & 57.95 & 31.95 & 19.80 & 65.78 & 83.50 & 38.38 & 52.57 & 47.00 \\
    base $R{=}3$         & 24.74 & 61.28 & 33.42 & 20.80 & \textbf{67.68} & 86.80 & 37.87 & 52.41 & 48.13 \\
    base $R{=}5$         & 26.88 & 62.08 & 33.92 & \textbf{23.20} & 67.52 & \textbf{87.10} & \textbf{40.02} & 52.57 & 49.16 \\
    DeepLoop $R{=}1$     & 25.26 & 57.15 & 31.90 & 21.00 & 67.14 & 82.80 & 38.89 & 51.30 & 46.93 \\
    DeepLoop $R{=}3$     & 26.88 & 61.41 & 33.91 & 21.00 & 67.46 & \textbf{87.10} & 39.76 & 53.83 & 48.92 \\
    DeepLoop $R{=}5$     & \textbf{28.24} & \textbf{62.67} & \textbf{34.64} & 22.60 & 67.41 & 87.00 & \textbf{40.02} & \textbf{54.70} & \textbf{49.66} \\
    \bottomrule
  \end{tabular}
\end{table}

\begin{figure}[ht]
  \centering
  \includegraphics[width=0.66\linewidth]{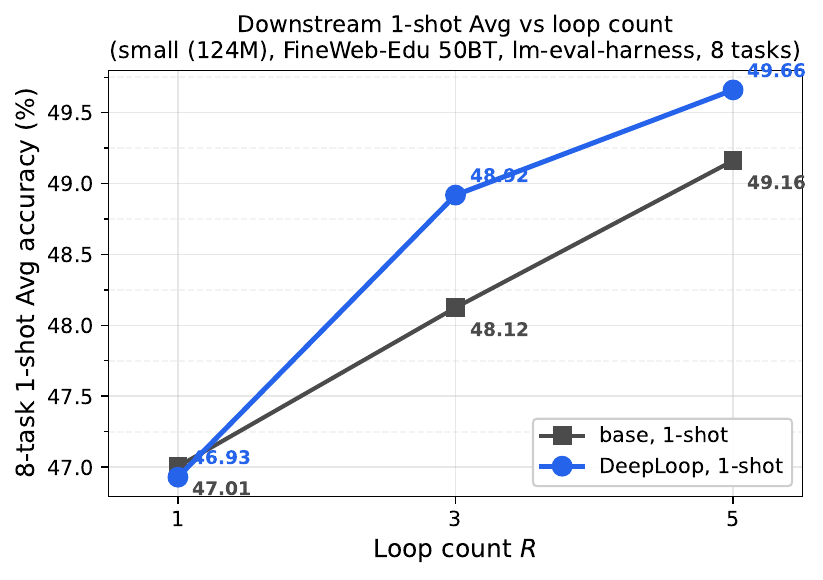}
  \caption{Eight-task 1-shot Avg accuracy against loop count $R$ for
    the GPT-2 small backbone, base versus DeepLoop. Numbers match
    the 1-shot Avg column of Table~\ref{tab:downstream-small}.}
  \label{fig:downstream-avg-small}
\end{figure}

\begin{figure}[ht]
  \centering
  \includegraphics[width=0.95\linewidth]{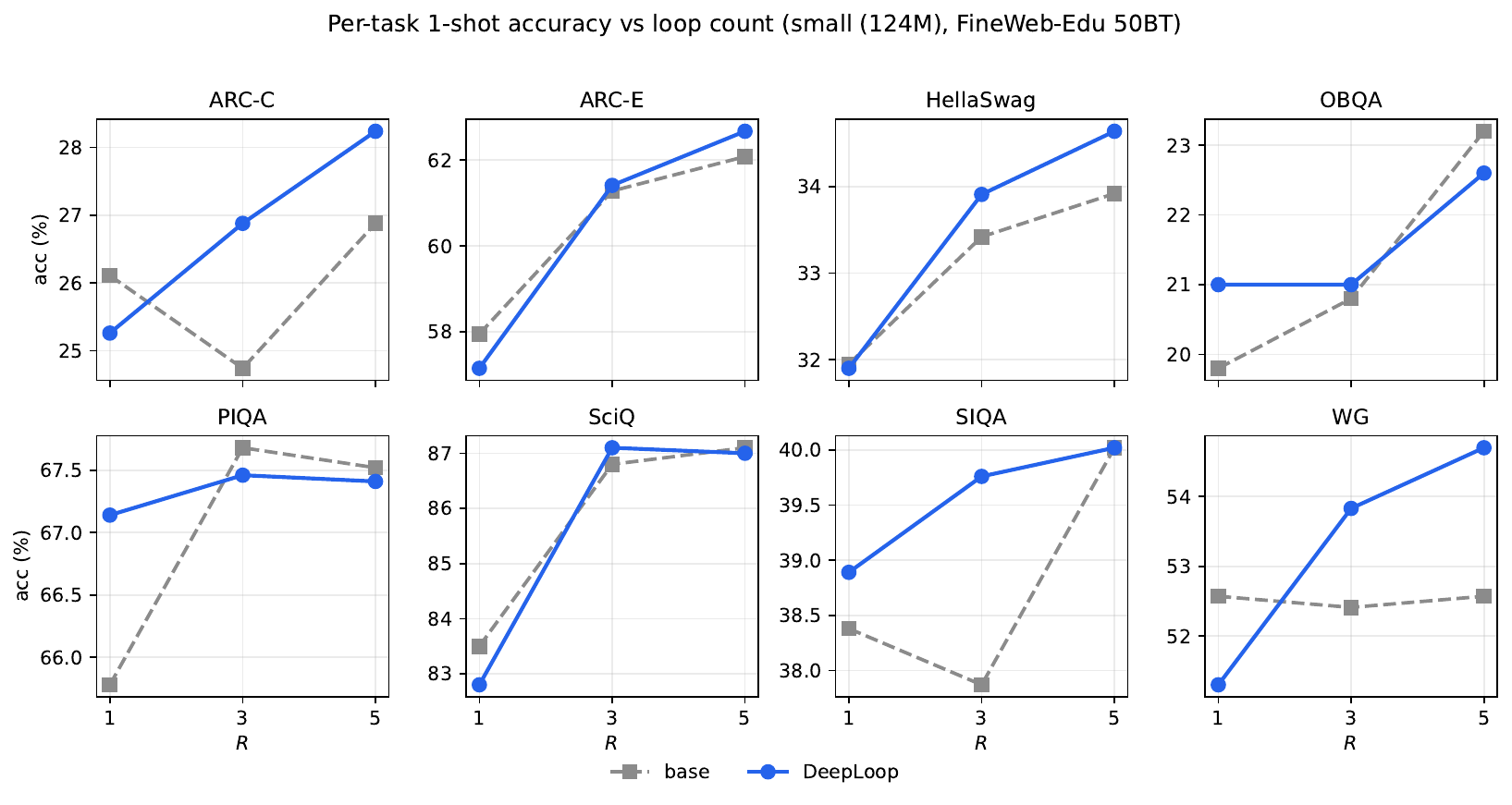}
  \caption{Per-task 1-shot accuracy against loop count $R$ for base
    versus DeepLoop on all eight tasks at the GPT-2 small scale.
    Numbers match Table~\ref{tab:downstream-small}.}
  \label{fig:downstream-per-task-small}
\end{figure}

The two methods are within $\pm 0.25$ points of each other at $R{=}1$
on both shot settings, and DeepLoop has higher Avg at $R{=}3$ (by
$+1.31$ and $+0.79$ points on 0-shot and 1-shot respectively) and at
$R{=}5$ (by $+0.15$ and $+0.50$). The headline small-scale cell is
DeepLoop $R{=}5$ 1-shot at $49.66\,\%$. Against the same-$R$ baseline
(base $R{=}5$ 1-shot, $49.16\,\%$), DeepLoop wins on four of the
eight tasks (ARC-C, ARC-E, HellaSwag, WinoGrande), is within $0.11$
points on three tasks (PIQA, SciQ, SIQA), and loses clearly only on
OpenBookQA.

\clearpage

\section{Empirical \texorpdfstring{$p$}{p}-sweep at fixed loop count}
\label{app:psweep}

Proposition~\ref{prop:exponent-threshold} predicts a $p{=}1/2$
training-stability threshold under the worst-case aligned regime.
We test this prediction directly with a single-axis sweep over the
exponent $p$ at fixed $R{=}3$ on the GPT-2 small backbone. Every sweep cell
uses the sandwich block of Eq.~(\ref{eq:deeploop-block}). All other
ingredients, tied embeddings, input-embedding RMSNorm, FineWeb-Edu
50BT data pipeline, batch size, learning rate, and seed protocol, are
identical to the DeepLoop runs in Table~\ref{tab:loop-valloss}; only the
residual
scaling $\alpha=(2N)^p$ and initialization gain $\beta=(8N)^{-p}$
change. We run a grid of seven values, $p\in\{0.30,0.35,0.40,0.45,0.50,
0.55,0.60\}$, with up to five seeds per value. Each cell is a 90-minute
run on $4\times$H200 141\,GB GPUs, which under our throughput reaches
roughly $2{,}700$ optimizer steps; we report validation loss at the
common comparison checkpoint of step~$2000$ where available, and use
the trailing-50-iter mean train loss as a surrogate for runs that do not reach a post-warmup validation step.

\begin{figure}[ht]
  \centering
  \includegraphics[width=0.62\linewidth]{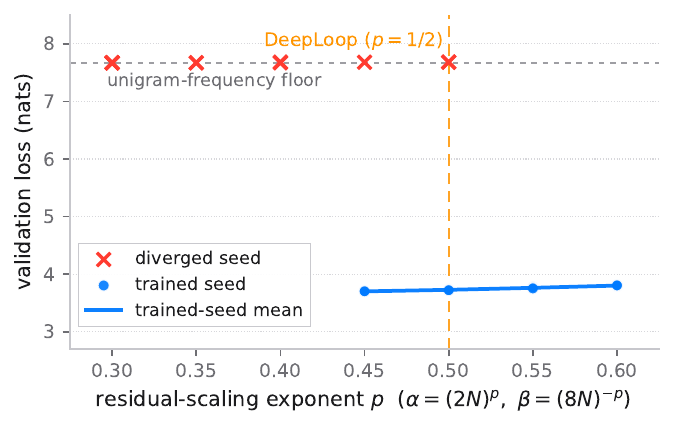}
  \caption{Validation-loss snapshot at step~$2000$ on FineWeb-Edu 50BT
    for the GPT-2 small backbone, $R{=}3$, sweeping the residual-scaling
    exponent $p\in\{0.30,\ldots,0.60\}$. Each marker is the most recent
    attempt for one $(p,\mathrm{seed})$ cell; several cells were
    relaunched, and the pooled per-seed escape counts over all attempts
    are reported in the text. Red
    crosses are attempts that do not escape the unigram-frequency floor
    ($\approx 7.67$\,nats); blue circles are attempts that train; the blue
    line connects the per-$p$ trained-attempt mean. Escape becomes more
    frequent as $p$ increases, and every seed trains at $p\ge 0.55$; the
    transition brackets the
    Proposition~\ref{prop:exponent-threshold} worst-case prediction of
    $p{=}1/2$.
    Conditional on convergence, larger $p$ gives a slightly higher loss
    (less aggressive learning).}
  \label{fig:p-sweep}
\end{figure}

The picture in Figure~\ref{fig:p-sweep} is consistent with the analysis.
Several $(p,\mathrm{seed})$ cells were launched more than once for
scheduling reasons; we count a seed as escaping if \emph{any} of its
attempts reaches a validation evaluation below the unigram floor
($\approx 7.67$\,nats) within the budget. Under this pooled count, the
per-seed escape fractions are $0/5$, $1/5$, $2/5$, $2/5$, $3/5$, $5/5$,
and $5/5$ for $p=0.30,\ldots,0.60$: escape is the minority outcome in
the clearly sub-threshold region $p\le 0.40$, remains unreliable through
$p=0.50$, and becomes universal from $p=0.55$. At this short horizon the
transition is gradual rather than sharp, and it brackets the $p{=}1/2$
worst-case prediction of Proposition~\ref{prop:exponent-threshold}
rather than pinpointing it. Second, among the plotted attempts that do
train, the mean validation loss at the same step rises monotonically
with $p$ ($3.70$ at $p{=}0.45$, $3.73$ at $p{=}0.50$, $3.76$ at
$p{=}0.55$, $3.80$ at $p{=}0.60$), confirming that more aggressive
scaling (smaller $p$) delivers a stronger learning signal whenever it
does not destabilize training. We adopt $p{=}1/2$ as the DeepLoop
default on the strength of the worst-case analysis; the sweep places
the empirical transition around it, with failures concentrated below it
and the loss penalty for exceeding it growing with $p$. Choices
$p<1/2$ produce lower trained-attempt mean loss when they do train, but
with at least $3/5$ of seeds failing to escape the floor at every
$p\le 0.45$, they are not safe defaults at this horizon.

The $p$-sweep is small in scope: it is at one scale (GPT-2 small), one
loop depth ($R{=}3$), and one optimizer-step budget. Whether the same
boundary at $p{=}1/2$ applies at larger $K$, at different normalization
placements, or at substantially longer training, is left to future work.

\clearpage

\section{Compute resources}
\label{app:compute}

All experiments were run on NVIDIA H200 141\,GB GPUs in a SLURM cluster.
GPT-2 small cells use $4\times$H200, and GPT-2 medium cells
use $8\times$H200. Per-cell wall-clock time is read from the corresponding
WandB run records; aggregate GPU-hours below are rounded to the nearest
hundred.

\paragraph{Per-cell cost.} Wall-clock time scales roughly linearly with
the loop count $R$:

\begin{center}
\small
\begin{tabular}{lrrrr}
  \toprule
  & $R{=}1$ & $R{=}3$ & $R{=}5$ & $R{=}7$ \\
  \midrule
  Small wall-clock (h)    & $9$  & $22$  & $35$  & $50$  \\
  Small GPU-hours / cell  & $37$ & $90$  & $145$ & $200$ \\
  \midrule
  Medium wall-clock (h)   & $12$ & $32$  & $52$  & $77$  \\
  Medium GPU-hours / cell & $95$ & $260$ & $420$ & $620$ \\
  \bottomrule
\end{tabular}
\end{center}

\paragraph{Aggregate cost.} The reported FineWeb-Edu language-modeling
matrix (Tables~\ref{tab:loop-valloss} and~\ref{tab:downstream-medium}: 16
cells $\{$base, DeepLoop$\}\times\{$small, medium$\}\times R\in\{1,3,5,7\}$)
totals approximately $\mathbf{3{,}700}$ H200 GPU-hours. Exploratory
experiments not included in the main paper, a range of
optimization-related variants and additional benchmark sweeps we
evaluated during the development of the method and ultimately did not
report, add approximately $\mathbf{7{,}000}$ GPU-hours.

The total project compute is therefore approximately $\mathbf{10{,}700}$
H200 GPU-hours.

\end{document}